\documentclass[preprint,3p]{elsarticle}

\usepackage{amssymb}
\usepackage[caption=false]{subfig}
\usepackage{graphicx}
\usepackage{supertabular,booktabs}
\usepackage{amsmath,amssymb}
\usepackage{algorithm}
\usepackage{algcompatible}
\usepackage{multirow}
\usepackage{multicol}
\usepackage{makecell}
\usepackage{diagbox}
\usepackage{hyperref}
\usepackage{color}
\usepackage{ntheorem}
\newenvironment{proof}{{\noindent\bf Proof}\quad}{\hfill $\square$\par}

\newtheorem{lemma}{Lemma}
\newtheorem{theorem}{Theorem}

\newtheorem{problem}{Problem}

\journal{Swarm and Evolutionary Computation}

\begin{document}

\begin{frontmatter}



\title{Error Analysis of Elitist Randomized Search Heuristics\tnoteref{label1}}
\tnotetext[label1]{A preliminary research on this topic was published in proceedings of BIC-TA 2019.}
\author[CHEN]{Cong Wang}
\author[CHEN]{Yu Chen\corref{cor}}
\address[CHEN]{School of Science, Wuhan University of Technology, Wuhan 430070, China}
\ead{ychen@whut.edu.cn }

\author[HE]{Jun He}
\address[HE]{School of Science and Technology, Nottingham Trent University, Nottingham NG11 8NS, UK}

\author[XIE]{Chengwang Xie\corref{cor}}
\address[XIE]{School of Computer and Information Engineering, Nanning Normal University, Nanning 530299, China}
\ead{chengwangxie@nnnu.edu.cn}
\cortext[cor]{Corresponding authors}

\begin{abstract}
When globally optimal solutions of complicated optimization problems cannot be located by evolutionary algorithms (EAs) in polynomial expected running time, the hitting time/running time analysis is not flexible enough to accommodate the requirement of theoretical study, because sometimes we have no idea on what approximation ratio is available in polynomial expected running time. Thus, it is necessary to propose an alternative routine for theoretical analysis of EAs.  To bridge the gap between theoretical analysis and algorithm implementation, in this paper we perform an error analysis where expected approximation error is estimated to evaluate performances  of randomized search heuristics (RSHs). Based on the Markov chain model of RSHs, the multi-step transition matrix can be computed by diagonalizing the one-step transition matrix, and a general framework for estimation of expected approximation errors is proposed. Case studies indicate that the error analysis works well for both uni- and multi-modal benchmark problems. It leads to precise estimations of approximation error instead of  asymptotic results on fitness values, which demonstrates its competitiveness to fixed budget analysis.
\end{abstract}


\end{frontmatter}

\section{Introduction}
In principle, evolutionary algorithms (EAs) could be employed solving a wide variety of optimization problems, which contributes to its popular application in scientific and engineering fields. However, their performances are significantly influenced by the mathematical characteristics of investigated problems. So, it is helpful to analyze fitness landscapes of the investigated problems~\cite{yang2016dynamic}, and then, design individualized strategies to accommodate them efficiently~\cite{li2019parameter,gong2020finding}.

Wise design of individualized strategies could be based on results of theoretical study. Theoretical analysis of randomized search heuristics (RSHs) was usually focused on estimation of the expected first hitting time (FHT) or the expected running time (RT), which quantify the needed evaluation budget to hit the global optimal solutions. A variety of theoretical routines were proposed for estimation of FHT/RT~\cite{he2001drift,droste2002analysis,he2003towards,doerr2012multiplicative,yu2014switch}, and massive theoretical results have been reported in the past years~\cite{oliveto2007time,zhou2009runtime,huang2009pheromone,friedrich2010approximating,sutton2014parameterized,doerr2017time,qian2017constrained}.

Although popular in theoretical analysis, estimation of FHT/RT does not make sense when RSHs are not anticipated to locate the global optimal solutions with satisfactory efficiency. For this case, the FHT/RT is usually exponential in problem size. To obtain polynomial FHT/RT of RSHs, a remedy is to investigate the approximation performance by taking an approximation set as the hitting destination of consecutive iterations~\cite{chen2011drift,yu2012approximation,lai2013performance,huang2014runtime,zhou2014approximation,xia2015analysis,xia2015approximation,peng2016approximation,yushan2016first}.

However,  to perform an analysis of approximation performance one must preset an approximation ratio in advance, which is sometimes unavailable if we are short of knowledge about  mathematical properties of the investigated problems. Inspired by the fact that in numerical experiments EAs are usually evaluated  by qualities of solutions obtained with given budget.  Jansen and Zarges proposed to estimate objective values by fixed-budget analysis~\cite{jansen2014performance}. Following this theoretical routine, Jansen and Zarges performed a theoretical evaluation of immune-inspired hypermutations~\cite{jansen2014reevaluating}, and Nallaperuma \emph{et al.} investigated performances of RSH on the traveling salesperson problem~\cite{nallaperuma2017expected}. {\color{red} Fixed budget analysis generates bound estimations of $f(X^{[t]})$ for given iteration budget $t$, which is not general and sometimes invalid for a large $t$. Moreover, it is performed by analysis tricks depending on properties of the investigated problems, and so,  general analysis frameworks are not easy to be obtained. For the same reason, the analysis is very complicated, and sometimes only asymptotic results can be achieved.}

Convergence rate, which is often employed assessing convergence speed of deterministic iteration algorithms, is also used in theoretical analysis {\color{red}of} RSH. Due to the stochastic iteration mechanism of RSH, the convergence rate is defined as $r^{[t]}=e^{[t]}/e^{[t-1]}$, where $e^{[t]}$ is the expected approximation error at generation $t$. By restricting the convergence rate under the condition $r^{[t]} \le \lambda <1$, Rudolph~\cite{rudolph1997convergence} proved that the sequence $\{e^{[t]}; t=0,1, \cdots\}$ converges in mean geometrically fast to $0$. Considering that numerical simulation of$r^{[t]}$ is unstable, He and Lin~\cite{he2016average} investigated the geometric average convergence rate (ACR) of  for binary-coded RSH, defined as
$
R^{[t]} = 1- \left( {e^{[t]}}/{e^{[0]}} \right)^{1/t}
$.
They estimated the lower bound on $R^{[t]}$ and proved if the initial population is randomly initialized, $R^{[t]}$ converges to an eigenvalue of the transition matrix associated with an EA. Recently, Chen and He~\cite{chen2019average} performed an ACR analysis for continuous RSH, which demonstrated a significantly different performance of RSH for continuous optimization problems.

Starting from the convergence rate $r^{[t]}$ or $R^{[t]}$, it is straightforward to get an exact expression of the approximation error: {\color{red} $$e^{[t]} =e^{[0]} \prod_{k=1}^{t}r^{[k]} \quad \mbox{or}\quad e^{[t]} =e^{[0]}(1-R^{[t]})^t.$$}
By estimating one-step convergence rate $r^{[t]}$ for any $t$, He \emph{et al.}~\cite{he2019unlimited1,he2019unlimited2} performed the unlimited budget analysis for approximation error of RSH, which demonstrated a general framework for estimation of approximation error for any computation budget. However, it is not trivial to derive the convergence rate dependent on $t$, and a general estimation of $r^{[t]}$ could lead to a very loose estimation of the expected approximation error.

By investigate the Markov chain model of EAs, He~\cite{he2016analytic} made a first attempt to  obtain an analytic expression of the approximation error for. He proved if
the transition matrix associated with an EA is an upper triangular matrix with distinct diagonal entries, the relative  error $e^{[t]}$ for any $t \ge 1$ is expressed by
$ e^{[t]} =\sum^{L}_{k=1}  c_k \lambda_k^{t-1},
$ where $\lambda_k$  are eigenvalues of the transition matrix and $c_k$ are coefficients. In accordance with this idea, He \emph{et al.}~\cite{he2018theoretical} proposed to compute $c_k$ and $\lambda_k^{t-1}$ by estimating $t$-th power of the transition matrix, and presented several mathematical routines depending on the properties of transition matrices.

{\color{red} As suggested by He \emph{et al.}~\cite{he2016analytic,he2018theoretical}, to compute expected approximation error it is necessary to confirm the coefficients $c_k$ and eigenvalues $\lambda_k$. As a first attempt, we investigated performances of RSH for the case that the status transition matrices can be computationally diagonalized, and estimated the expected approximation error for arbitrary iteration budget~\cite{wang2019estimating}. However, when the bitwise mutation is employed, the elitist selection would generate a Markov chain model with an upper triangular transition matrix, and it is difficult to get analytic expression of its $t$-th power. In this study, we extend our research to investigate more complicated cases. When a global search strategy is implemented for multi-modal problems, the transition matrix is much more complicated. Thus, we construct auxiliary search processes modelled by bi-diagonal transition matrices, and analyze performances of an elitist RSH by computing expected approximation error of the auxiliary search process. In this way, a general framework to estimate approximation error of elitist RSH is available, by which we can ge an exact estimation of approximation error instead of the asymptotic results by fixed-budget analysis.} Rest of this paper is organized as follows. Section \ref{SecPre} presents some preliminaries. Section \ref{SecGen} proposes general results on the approximation error, and case studies are performed in Section \ref{SecCase}. Moreover, applicability of the theoretical framework is further verified in Section \ref{SecKnap} by investigating an instance of the knapsack problem. Finally, Section \ref{SecCon} concludes this paper.

\section{Preliminaries}\label{SecPre}
In this paper, we consider the maximization problem
\begin{equation}\label{CO}
  \max\quad f(\mathbf{x}),\quad\mathbf{x}=(x_1,\dots,x_n)\in\{0,1\}^n.
\end{equation}
Denote its optimal solution as $\mathbf{x}^*$, and the corresponding objective value $f^*$. Then, quality of a solution $\mathbf x$ can be evaluated by its approximation error $e(\mathbf x)=|f(\mathbf x)-f^*|$. For error analysis, an elitist RSH described in Algorithm \ref{alg1} is investigated in this paper. When the one-bit mutation is employed, it is called a \emph{random local search (RLS)}; if the bitwise mutation is used, it is named as a \emph{(1+1)EA}.

\begin{algorithm}[ht]
\caption{Elitist Randomized Search Heuristics}
\label{alg1}
\begin{algorithmic}[1]
\STATE  counter $t = 0$;
\STATE randomly  initialize a solution $\mathbf{x}_{0}$;
\WHILE{the stopping criterion is not satisfied}
\STATE  generate a new candidate solution $\mathbf{y}_{t}$ from $\mathbf{x}_{t}$ by mutation;
\STATE set individual $\mathbf{x}_{t+1}=\mathbf{y}_{t}$ if $f(\mathbf{y}_{t})>f(\mathbf{x}_{t})$; otherwise, let $\mathbf{x}_{t+1}=\mathbf{x}_{t}$;
\STATE $t= t+1$;
\ENDWHILE
\end{algorithmic}
\end{algorithm}

The population sequence $\{\mathbf{x}_t,t=0,1,\dots\}$ of RLS/(1+1)EA is a \emph{Homogeneous Markov Chain (HMC)}. Classify the solution set into $L+1$ mutually disjoint subset $\mathcal{X}_0,\mathcal{X}_1,\dots, \mathcal{X}_L$, where solutions in $\mathcal{X}_i$ have identical approximation error $e_i$ satisfying
\begin{equation}\label{Sec1MonoError}
  0= e_0\le e_1\le\dots\le e_L.
\end{equation}
If $\mathbf{x}\in\mathcal{X}_i$, it is called \emph{at the status $i$}. Status $0$, consisting of globally optimal solutions, is called the \emph{optimal status}, and other statuses are the \emph{non-optimal statuses}. Then, $\{\mathbf{x}_t,t=0,1,\dots\}$ is a discrete HMC with $L+1$ available statuses, and the
transition probability matrix is  $\mathbf{\tilde{R}}=(r_{i,j})_{(L+1)\times (L+1)}$, where
\begin{equation*}
  r_{i,j}=\Pr\{\mathbf{x}_{t+1}\in\mathcal{X}_i|\mathbf{x}_{t}\in\mathcal{X}_j\},\quad i,j=0,\dots,L.
\end{equation*}
While the elitist RSH is employed solving a maximization problem with the error vector
  $\mathbf{\tilde{e}}=(e_0,\dots,e_l)'$,
initialization  of solutions $\mathbf{x_0}$ would generate an initial status distribution
  $\mathbf{\tilde{p}}^{[0]}=(p_{0}^{[0]},p_{1}^{[0]},\dots,p_{L}^{[0]})'$.
Then, after $t$ generations, we get the status distribution
\begin{equation*}
  \mathbf{\tilde{p}}^{[t]}=(p_{0}^{[t]},p_{1}^{[t]},\dots,p_{L}^{[t]})'=\mathbf{\tilde{R}}^t\mathbf{\tilde{p}}^{[0]},
\end{equation*}
and the expected approximation error can be confirmed as
\begin{equation}\label{Sec1ExpError}
  e^{[t]}=\mathbf{\tilde{e}}'\mathbf{\tilde{R}}^t\mathbf{\tilde{p}}^{[0]}.
\end{equation}

Since the elitist selection is employed, the transition matrix $\mathbf{\tilde{R}}$ is upper triangular, and it can be partitioned as
\begin{equation}\label{TranMatrix}
  \mathbf{\tilde{R}}=\left(
               \begin{array}{cc}
                 1 & \mathbf{r}_0 \\
                 \mathbf{0} & \mathbf{R} \\
               \end{array}
             \right),
\end{equation}
where $\mathbf{r}_0=(r_{0,1},r_{0,2},\dots,r_{0,L})$, $\mathbf{0}=(0,\dots,0)'$,
\begin{equation}\label{Sec1R}
  \mathbf{R}=\left(
               \begin{array}{ccc}
                 r_{1,1} & \dots & r_{1,L} \\
                  & \ddots & \vdots \\
                  &  & r_{L,L} \\
               \end{array}
             \right).
\end{equation}
The following lemma demonstrates that the expected approximation error of RSH is independent on initial distribution of the optimal status and the one-step transition probability  from non-optimal statuses to the optimal one.
\begin{lemma}\label{Sec1L1}
Let $\mathbf{\tilde{e}}=(e_0,e_1,\dots, e_L)$ and $\mathbf{\tilde{r}}=(r_0,r_1,\dots,r_L)'$ be non-negative vectors. If $e_0=0$, it holds that
\begin{equation*}
  \mathbf{\tilde{e}}'\mathbf{\tilde{R}}^t\mathbf{\tilde{r}}=\mathbf{e}'\mathbf{R}^t\mathbf{r},\quad,\forall\,\,t\in\mathbb Z^+,
\end{equation*}
where $\mathbf{e}=(e_1,\dots, e_L)$, $\mathbf{r}=(r_1,\dots,r_L)'$, $\mathbf{R}$ and $\mathbf{\tilde{R}}$ confirmed by (\ref{TranMatrix}) and (\ref{Sec1R}), respectively.
\end{lemma}
\begin{proof}
The proof is trivial, and we can complete it by the following deduction.
\begin{equation*}
  \mathbf{\tilde{e}}'\mathbf{\tilde{R}}^t\mathbf{\tilde{r}}=(0, \mathbf{e}') \left(
               \begin{array}{cc}
                 1 & \mathbf{r}_0 \\
                 \mathbf{0} & \mathbf{R} \\
               \end{array}
             \right)(r_0, \mathbf{r}')'=\mathbf{e}'\mathbf{R}^t\mathbf{r},\quad,\forall\,\,t\in\mathbb Z^+.
\end{equation*}
\end{proof}

{\color{red}Then, estimation of the expected approximation error of RSHs relys on computation of $\mathbf{R}^t$. If $\mathbf{R}$ has $L$ distinct diagonal elements, it can be theoretically diagonalized~\cite{lay2003linear}. That is, there exist an invertible matrix $\mathbf{P}$ such that $\mathbf{R}^t=\mathbf{P}\Lambda ^t\mathbf{P}^{-1}$, where $\Lambda=diag\{r_{1,1},r_{2,2},\dots,r_{L,L}\}$. However, deduction of the transformation matrix $\mathbf{P}$ is sometimes difficult, which depends on the distribution of non-zero elements in $\mathbf{R}$.
According to the distribution of non-zero elements, we can classify the searching process of elitist RSH into three categories.
\begin{enumerate}
\item \textbf{Diagonal Search}: If the transition submatrix is a diagonal matrix
\begin{equation}\label{DiagTran}
  \mathbf{R}_D=diag\{r_{1,1},r_{2,2},\dots,r_{L,L}\},
\end{equation}
we call that an EA generates a \emph{diagonal search}.
  \item \textbf{Bi-Diagonal Search}: If the transition submatrix is a bi-diagonal matrix
  \begin{equation}\label{BiDiagTran}
  \mathbf{R}_{BD}=\left(
             \begin{array}{ccccc}
               r_{1,1} & r_{1,2} &  &  &\\
                & r_{2,2} & r_{2,3} &  &\\
                &  & \ddots &\ddots  &\\
                &  &  &r_{L-1,L-1}  &r_{L-1,l} \\
                &  &  &  & r_{L,L}\\
             \end{array}
           \right),
\end{equation}
we call that an EA generates a \emph{bi-diagonal search}.
  \item \textbf{Elisit Search}: If the transition submatrix is an upper triangular matrix
  \begin{equation}\label{UppTran}
  \mathbf{R}_{E}=\left(
             \begin{array}{ccccc}
               r_{1,1} & r_{1,2} & r_{1,3} & \dots &r_{1,L}\\
                & r_{2,2} & r_{2,3} & \dots &r_{2,L}\\
                &  & \ddots &\ddots  &\vdots\\
                &  &  &r_{L-1,L-1}  &r_{L-1,L} \\
                &  &  &  & r_{L,L}\\
             \end{array}
           \right),
\end{equation}
we call that an EA generates an \emph{elitist search}.
\end{enumerate}

By definition of status we know that an elitist EA must generates an elitist search. A diagonal search is necessarily a bi-diagonal search, and a bi-diagonal search is an elitist one. The easiest task is to estimate the approximation error of an diagonal search, and the hardest task is that of an elitist search.}


\section{General Results on Estimation of Expected Approximation Error}\label{SecGen}
For an diagonal search, computation of the $t$-th power of transition submatrix is a trivial task because the $\mathbf{R}_D$ is diagonal. For the submatrix $\mathbf{R}_{BD}$ of a bi-diagonal search, we can also deduce the analytic forms of the transformation matrix $\mathbf{P}$ and its inverse $\mathbf{P}^{-1}$, and then get the analytic form of the $t$-th power. However, it is difficult to get the precise expression for the $t$-th power of a general upper-triangular matrix. Thus, we would construct an auxiliary bi-diagonal search that converges more slowly than the elitist one, and get an upper bound for the expected approximation error of an elitist search.
\subsection{Expected Approximation Error of a Diagonal Search}

\begin{theorem}\label{Th_D}
For the error vector $\mathbf{e}=(e_1,\dots,e_L)'$ of non-optimal statuses with $e_{i}\le e_{i+1}$, $\forall\,i=1,\dots, L-1$,
\begin{equation*}
  e_D^{[t]}=\sum_{i=1}^Lr_{i,i}^te_ip^{[0]}_i,
\end{equation*}
where $\mathbf{p}^{[0]}=(p^{[0]}_1,\dots,p^{[0]}_L)'$.
\end{theorem}
\begin{proof}
Because $$\mathbf{R}_D^t=\left(diag\{r_{1,1},r_{2,2},\dots,r_{L,L}\}\right)^t=diag\{r_{1,1}^t,r_{2,2}^t,\dots,r_{L,L}^t\},$$
Lemma \ref{Sec1L1} implies that
$$e_D^{[t]}=\mathbf{e}'\mathbf{R_D}^t\mathbf{p}^{[0]}=\mathbf{e}'\left(diag\{r_{1,1}^t,r_{2,2}^t,\dots,r_{L,L}^t\}\right)\mathbf{p}^{[0]}=\sum_{i=1}^Lr_{i,i}^te_ip^{[0]}_i.$$
\end{proof}

\subsection{Expected Approximation Error of a Bi-Diagonal Search}
To estimate the expected approximation error of a bi-diagonal search, it is essential to compute the $t$-th power of a bi-diagonal matrix $\mathbf R$. While $\mathbf R$ has $n$ distinct eigenvalues, it can be diagonalized by similarity transformation.
\begin{lemma}\label{Sec2L2}
\cite{lay2003linear} If an $L\times L$ matrix $\mathbf{A}$ has $L$ distinct eigenvalues $\lambda_1$, $\lambda_2$,...,$\lambda_L$, it can be diagonalized as
\begin{equation}\label{diag}
  \boldsymbol{\Lambda}=\mathbf{P}^{-1}\mathbf{A}\mathbf{P}.
\end{equation}
Here, $\boldsymbol{\Lambda}=diag\{\lambda_1,\dots,\lambda_L\}$, $\mathbf{P}=(\mathbf{p}_1,\dots,\mathbf{p}_L)$, where $\mathbf{p}_i$ is the corresponding eigenvector of $\lambda_i$ with $$\mathbf{A}\mathbf{p}_i=\lambda_i\mathbf{p}_i,\quad i=1,\dots,L.$$
\end{lemma}

Note that equation (\ref{diag}) is equivalent to $\mathbf{A}=\mathbf{P}\boldsymbol{\Lambda}\mathbf{P}^{-1}$.
Then,
\begin{equation}\label{Diagonalization}
\mathbf{A}^t=\left(\mathbf{P}\boldsymbol{\Lambda}\mathbf{P}^{-1}\right)^t=\mathbf{P}\boldsymbol{\Lambda}^t\mathbf{P}^{-1},
\end{equation}
which can be confirmed by computing the matrix $\mathbf{P}$ and its inverse $\mathbf{P}^{-1}$.

\begin{lemma}\label{Sec2L3}
If the bi-diagonal matrix $\mathbf{R}_{BD}$ confirmed by  (\ref{BiDiagTran}) has $L$ distinct diagonal elements,
\begin{align*}
  \mathbf{R}_{BD}^t=\sum_{j=1}^{L}\lambda_j^t\mathbf{p}_j\mathbf{q}'_j,
\end{align*}
where
\begin{align}
  \mathbf{p}_{j}&=\left(\prod_{k=1}^{j-1}\frac{r_{k,k+1}}{r_{j,j}-r_{k,k}},\prod_{k=2}^{j-1}\frac{r_{k,k+1}}{r_{j,j}-r_{k,k}},\dots,\frac{r_{j-1,j}}{r_{j,j}-r_{j-1,j-1}},1,0,\dots,0\right)', \label{Eigen1} \\
  \mathbf{q}'_{j}&=\left(0,\dots,0,1,\frac{r_{j,j+1}}{r_{j,j}-r_{j+1,j+1}},\prod_{k=j+1}^{j+2}\frac{r_{k-1,k}}{r_{j,j}-r_{k,k}},\dots,\prod_{k=j+1}^{L}\frac{r_{k-1,k}}{r_{j,j}-r_{k,k}}\right),\label{ColTransMatrixIn}\\
  &\qquad\qquad\qquad\qquad\qquad\qquad\quad\quad\quad\quad\quad\quad\quad\quad\quad\quad\quad\quad\quad\quad\quad\quad\quad j=1\dots,L.\nonumber
\end{align}
\end{lemma}
\begin{proof}
If the upper triangular matrix $\mathbf{R}_{BD}$ has $L$ distinct diagonal elements,  it has $L$ distinct eigenvalues $$\lambda_i=r_{i,i},\quad i=1,\dots,L.$$
Then, Lemma \ref{Sec2L2} applies and $\mathbf{R}_{BD}$ can be diagonalized as
\begin{equation}\label{SimTrans}
  \mathbf {P}^{-1}\mathbf R_{BD}\mathbf  P=\boldsymbol{\Lambda}=diag\{\lambda_1,\lambda_2,\dots,\lambda_L\},
\end{equation}
where $\mathbf{P}=(\mathbf{p}_1,\dots,\mathbf{p}_L)$,
\begin{equation}\label{Eigen}
  \mathbf R_{BD}\mathbf p_{j}=\lambda_j\mathbf p_{j},\,\mathbf p_{j}\neq\mathbf 0,\quad j=1,\dots,L.
\end{equation}

Denote $\mathbf{p}_{j}=(p_{1,j},\dots,p_{L,j})'$. Equation (\ref{Eigen}) indicates that
\begin{equation*}
\left\{\begin{array}{l}
  r_{i,i}p_{i,j}+r_{i,i+1}p_{i+1,j}=r_{j,j}p_{i,j},\quad i=1,\dots,L;\\
  r_{L,L}p_{L,j}=r_{j,j}p_{L,j}.
  \end{array}\right.\quad j=1,2,\dots,L.
\end{equation*}
Note that $r_{i,i}\neq r_{j,j}$ when $i\neq j$. Thus, for the eigenvalue $\lambda_j$ we can obtain an corresponding eigenvector $\mathbf{p}_{j}=(p_{1,j},\dots,p_{L,j})'$  confirmed by
\begin{equation*}
p_{i,j}=\left\{\begin{aligned} & 0, && \mbox{ if } i>j;\\
& 1, && \mbox{ if } i=j;\\
& \prod_{k=i}^{j-1}\frac{r_{k,k+1}}{r_{j,j}-r_{k,k}}, && \mbox{ if }i<j;
\end{aligned}\right.\quad j=1,2,\dots,L.
\end{equation*}
That is,
\begin{align*}
 \mathbf{p}_{j}=\left(\prod_{k=1}^{j-1}\frac{r_{k,k+1}}{r_{j,j}-r_{k,k}},\prod_{k=2}^{j-1}\frac{r_{k,k+1}}{r_{j,j}-r_{k,k}},\dots,\frac{r_{j-1,j}}{r_{j,j}-r_{j-1,j-1}},1,0,\dots,0\right)',\quad j=1\dots,L.
\end{align*}

Denote  $\mathbf{Q}=\mathbf{P}^{-1}=(\mathbf{q}_1,\dots,\mathbf{q}_L)'$, where $\mathbf{q}_j=(q_{1,j},\dots,q_{L,j})'$. Because $\mathbf Q$ is the inverse matrix of $\mathbf P$, it is upper triangular, and its diagonal elements are inverse of the corresponding diagonal elements of $\mathbf P$. That is to say,
\begin{equation}\label{E2}
q_{j,j}=1,\quad j=1,\dots,L.
\end{equation}
By equation (\ref{SimTrans}) we know that
\begin{equation*}
(\mathbf R'_{BD}\mathbf{q}_1,\dots,\mathbf R'_{BD}\mathbf{q}_L)'=\mathbf Q\mathbf R_{BD}=\boldsymbol\Lambda \mathbf Q=(\lambda_1\mathbf{q}_1,\dots,\lambda_L\mathbf{q}_L)',
\end{equation*}
which implies that
\begin{equation*}
    r_{j-1,j}q_{i,j-1}+r_{j,j}q_{i,j}=r_{i,i}q_{i,j}, \quad i=1,\dots,L,\,\,j=1,\dots,L.
\end{equation*}
Combining it with equation (\ref{E2}), we know
\begin{equation*}
\left\{
\begin{array}{ll}
  q_{i,j}=0;&i=1,\dots,j-1;\\
  q_{i,j}=1;&i=j;\\
  q_{i,j}=\prod\limits_{k=i+1}^{j}\frac{r_{k-1,k}}{r_{i,i}-r_{k,k}}; &i=j+1,\dots, L,
  \end{array}\quad j=1,\dots,L.
  \right.
\end{equation*}
Then,
$\mathbf{Q}=\left(\mathbf{q}_1,\dots,\mathbf{q}_L\right)'$,
where
\begin{equation*}
\mathbf{q}'_{j}=\left(0,\dots,0,1,\frac{r_{j,j+1}}{r_{j,j}-r_{j+1,j+1}},\prod_{k=j+1}^{j+2}\frac{r_{k-1,k}}{r_{j,j}-r_{k,k}},\dots,\prod_{k=j+1}^{n}\frac{r_{k-1,k}}{r_{j,j}-r_{k,k}}\right),j=1\dots,L.
\end{equation*}

Finally, from (\ref{Diagonalization}) we conclude that
\begin{equation*}
  \mathbf{R}_{BD}^t=\mathbf{P}\boldsymbol{\Lambda}^t\mathbf{P}^{-1}=\mathbf{P}\boldsymbol{\Lambda}^t\mathbf{Q}
  =(\mathbf{p}_1,\dots,\mathbf{p}_L)diag\{\lambda_1^t,\dots,\lambda_L^t\}(\mathbf{q}_1,\dots,\mathbf{q}_L)'
  =\sum_{j=1}^{L}\lambda_j^t\mathbf{p}_j\mathbf{q}'_j.
\end{equation*}
\end{proof}

Then, the expected approximation error of a bi-diagonal search characterized by (\ref{BiDiagTran}) can be confirmed by the following theorem.

\begin{theorem}\label{Th_BD}
For non-optimal statuses, denote $\mathbf{e}=(e_1,\dots,e_L)'$ and  $\mathbf{p}^{[0]}=(p^{[0]}_1,\dots,p^{[0]}_n)'$. It holds that
\begin{equation*}
  e_{BD}^{[t]}=\sum_{j=1}^{L}\lambda_j^t\left(\mathbf{e}'\mathbf{p}_j\right)\left(\mathbf{q}_j'\mathbf{p}^{[0]}\right),
\end{equation*}
where $\mathbf{p}_j$ and $\mathbf{q}_j$ are confirmed by (\ref{Eigen1}) and (\ref{ColTransMatrixIn}), respectively.
\end{theorem}
\begin{proof}
Applying Lemmas \ref{Sec1L1} and \ref{Sec2L3}, we know that
\begin{align*}
  e_{BD}^{[t]}=\mathbf{e}'\mathbf{R}_{BD}^t\mathbf{p}^{[0]}=\mathbf{e}'\left(\sum_{j=1}^{L}\lambda_j^t\mathbf{p}_j\mathbf{q}_j'\right)\mathbf{p}^{[0]}=\sum_{j=1}^{L}\lambda_j^t\left(\mathbf{e}'\mathbf{p}_j\right)\left(\mathbf{q}_j'\mathbf{p}^{[0]}\right).
\end{align*}
\end{proof}

\subsection{Expected Approximation Error of an Elitist Search}
When the transition matrix $\mathbf{\tilde{R}}$ is upper triangular,  we would like to estimate not the precise expression but an upper bound of the approximation error. For an elitist search characterized by (\ref{UppTran}), this idea could be realized by constructing an auxiliary bi-diagonal search that converges more slowly than the original one.

\begin{lemma}\label{Sec3L4}
\cite{he2018theoretical} Provided that transition matrices $\mathbf{\tilde{R}}=(r_{i,j})_{(L+1)\times(L+1)}$ and $\mathbf{\tilde{S}}=(s_{i,j})_{(L+1)\times(L+1)}$ are upper triangular. If
\begin{align}
\label{conC1}
&s_{j,j}  \ge  r_{j,j}, &&\textrm{for any } j,
\\
\label{conC2}
&\textstyle \sum^{i-1}_{l=0} (r_{l,j}-s_{l,j}) \ge 0 , &&\textrm{for  any }i<j,
\\
\label{conC3}
&\textstyle  \sum^{i}_{l=0} ( s_{l,j-1}- s_{l,j})\ge 0 , &&\textrm{for any }i<j-1,
\end{align}
it holds that
\begin{equation*}
  \mathbf{T}(\mathbf{\tilde{R}})^t\le \mathbf{T}(\mathbf{\tilde{S}})^t,\quad \forall\,\,t\in\mathbb Z^+,
\end{equation*}
where
\begin{equation*}
  T=\left(\begin{array}{ccc}
      1 & \dots & 1 \\
       & \ddots & \vdots \\
       &   & 1
    \end{array}\right).
\end{equation*}
\end{lemma}

Construct an auxiliary search characterized by $\mathbf{\tilde{S}}$ in Lemma \ref{Sec3L4}, we can get the upper bound of approximation error for the elitist search.
\begin{theorem}\label{Th_ES}
Let $\mathbf{e}'=(e_1,\dots,e_L)$, $e_{i}\le e_{i+1}$, $i=1,\dots,L-1$, and $\mathbf{r}$ be a nonnegative $L$-dimensional vector. If transition matrices $\mathbf{\tilde{R}}$ and $\mathbf{\tilde{S}}$ satisfy conditions (\ref{conC1})-(\ref{conC3}), it holds that
  $$\mathbf{e}'\mathbf{R}^t\mathbf{r}\le \mathbf{e}'\mathbf{S}^t\mathbf{r}.$$
where $\mathbf{R}$ and $\mathbf{S}$ are the transition submatrices of $\mathbf{\tilde{R}}$ and $\mathbf{\tilde{S}}$, respectively.
\end{theorem}
\begin{proof}
Given vector $\mathbf{e}=(e_1,\dots,e_L)'$ with $e_{i}\le e_{i+1}$, $i=1,\dots,L-1$, and a nonnegative vector $\mathbf{r}$, we know that
\begin{align}
  &\mathbf{e}'\mathbf{R}^t\mathbf{r}-\mathbf{e}'\mathbf{S}^t\mathbf{r} =\mathbf{e}'\left(\mathbf{R}^t-\mathbf{S}^t\right)\mathbf{r}\nonumber\\
  =&(e_1,e_2-e_1,\dots,e_L-e_{L-1})\mathbf{T}\left(\mathbf{R}^t-\mathbf{S}^t\right)\mathbf{r}\nonumber\\
  \le & (e_1,e_2-e_1,\dots,e_L-e_{L-1})\left(\mathbf{T}\mathbf{R}^t-\mathbf{T}\mathbf{S}^t\right)\mathbf{r}.\nonumber
\end{align}
Then, Lemma \ref{Sec3L4} implies that
$$\mathbf{e}'\mathbf{R}^t\mathbf{r}-\mathbf{e}'\mathbf{S}^t\mathbf{r}\le 0,$$
that is, $\mathbf{e}'\mathbf{R}^t\mathbf{r}\le \mathbf{e}'\mathbf{S}^t\mathbf{r}$.
\end{proof}

By Theorem \ref{Th_ES}, we can set $\mathbf{r}$ as the initial distribution vector of non-optimal statuses to estimate the upper bound of expected approximation errors. To obtain a tight upper bound of the expected approximation error, the method for construction of such an auxiliary bi-diagonal search is problem-dependent.

\section{Case Study}\label{SecCase}
In this section, we would like to perform several case studies to demonstrate the feasibility of analysis routines proposed in Section \ref{SecGen}. For comparison, both the RLS and the (1+1)EA are considered for maximization of the following benchmark problems.
\begin{problem}\label{P1}(\textbf{OneMax})
\begin{equation*}
  \max f(\mathbf x)=\sum_{i=1}^nx_i, \quad\mathbf x=(x_1,\dots,x_n)\in \{0,1\}^n.
\end{equation*}
\end{problem}

\begin{problem}\label{P2}(\textbf{Peak})
\begin{equation*}
  \max f(\mathbf x)=\prod_{i=1}^{n}x_i,\quad\mathbf x=(x_1,\dots,x_n)\in \{0,1\}^n.
\end{equation*}
\end{problem}

\begin{problem}\label{P3}(\textbf{Deceptive Problem})
\begin{equation*}
  \max f(\mathbf x)=\left\{\begin{aligned}& \sum_{i=1}^nx_i, && \mbox{if }\sum_{i=1}^nx_i>n-1,\\ & n-1-\sum_{i=1}^nx_i, && \mbox{otherwise.} \end{aligned}\right.\quad\mathbf x=(x_1,\dots,x_n)\in \{0,1\}^n.
\end{equation*}
\end{problem}

\subsection{The OneMax Problem}
Objective value of the OneMax problem is the number of 1-bits in the bit-string, and the approximation error is number of 0-bits. Thus, the solution space can be divided into $n+1$ statuses labeled by their approximation errors. That is,  $$\mathbf{\tilde{e}}=(0,1,2,\dots,n)'.$$
Correspondingly, the initial distribution of status generated by random initialization is $$\mathbf{\tilde{p}}^{[0]}=(C_n^0/2^n,C_n^1/2^n,C_n^2/2^n,\dots,C_n^n/2^n)'.$$ Then, for non-optimal statues, the error vector and initial distribution are
\begin{align}\label{ErOneMax}
  \mathbf{e}=(1,2,\dots,n)',
\end{align}
and
\begin{align}\label{DiOneMax}
  \mathbf{p}^{[0]}=(C_n^1/2^n,C_n^2/2^n,\dots,C_n^n/2^n)'.
\end{align}
The expected approximation error of RLS is  given by the follow theorem.

\begin{theorem}\label{Th_RLS_OneMax}
The expected approximation error of RLS for the OneMax problem is $\frac{n}{2}\left(1-\frac{1}{n}\right)^t$.
\end{theorem}
\begin{proof}
Combining the one-bit mutation with elitist selection, RLS transfer from status $j$ to $j-1$ with probability $j/n$; otherwise, its individual status keeps unchanged. Thus, application of RLS on the unimodal OneMax problem generates a bi-diagonal search, the transition submatrix of which is
\begin{equation}\label{TrOneRLS}
  \mathbf{R}=(r_{i,j})_{n\times n}=\left(
                 \begin{array}{ccccc}
                   1-1/n & 2/n &   &   &   \\
                     & 1-2/n & 3/n &   &   \\
                     &   & \ddots & \ddots &   \\
                     &   &   & 1/n & 1 \\
                     &   &   &   & 0 \\
                 \end{array}
               \right).
\end{equation}

It is trivial to check that conditions of Theorem \ref{Th_BD} hold, and we know that
\begin{equation}\label{err1}
  e^{[t]}=\mathbf{e}'\mathbf{R}^t\mathbf{p}^{[0]}=\mathbf{e}'\left(\sum_{j=1}^{n}\lambda_j^t\mathbf{p}_j\mathbf{q}'_j\right)\mathbf{p}^{[0]}=\sum_{j=1}^{n}\lambda_j^t\left(\mathbf{e}'\mathbf{p}_j\right)\left(\mathbf{q}_j'\mathbf{p}^{[0]}\right),
\end{equation}
where $\lambda_j=r_{j,j}$, $\mathbf{p}_j$ and $\mathbf{q}_{j}$ are defined by (\ref{Eigen1}) and (\ref{ColTransMatrixIn}), respectively.
Substituting (\ref{ErOneMax}), (\ref{DiOneMax}), (\ref{temp1}) and (\ref{temp2}) to (\ref{err1}), we conclude that
  $$e^{[t]}=\sum_{j=1}^{n}\lambda_j^t\left(\mathbf{e}'\mathbf{p}_j\right)\left(\mathbf{q}'_j\mathbf{p}^{[0]}\right)=\lambda_1^t(\mathbf{e}'\mathbf{p}_1)(\mathbf{q}'_1\mathbf{p}^{[0]})=\frac{n}{2}\left(1-\frac{1}{n}\right)^t.$$
\end{proof}

The same results about performance of RLS on the OneMax problem have also been reported in \cite{jansen2014performance,he2019unlimited1}. Jansen and Zarges get this results by the law of total probability~\cite{jansen2014performance}, and He \emph{et al.} get it with the help of the constant convergence rate~\cite{he2019unlimited1}. {\color{red} If the iteration budget $t=an^k$, $(a>0,k\in\mathbb{Z}^+)$,
we know
$e^{[t]}=\frac{n}{2}\left(1-\frac{1}{n}\right)^{an^k}$.
Because
\begin{equation*}
  \frac{1}{2e}<\left(1-\frac{1}{n}\right)^{n}<\frac{1}{2},\quad n\ge 2,
\end{equation*}
it holds that
\begin{equation*}
  \frac{n}{2}\left((2e)^a\right)^{-n^{k-1}}<e^{[t]}<\frac{n}{2}\left(2^a\right)^{-n^{k-1}},
\end{equation*}
which indicates the  asymptotic expected approximation error is $C^{-\Omega(n^{k-1})}$ for some constant $C>1$.}

There are several ways to analyze performance of RLS on OneMax because the distribution of status transition is very simple. When the (1+1)EA is employed, the analyzing process presented in \cite{jansen2014performance} is very complicated. However, estimation of expected approximation error for the (1+1)EA is just a simple implementation of the analysis framework proposed in Section \ref{SecGen}.

\begin{theorem}\label{Th_EA_OneMax}
The expected approximation error of (1+1)EA for the OneMax problem satisfies
\begin{equation*}
  e^{[t]}\le\frac{n}{2}\left(1-\frac{1}{ne}\right)^t.
\end{equation*}
\end{theorem}
\begin{proof}
Since the bitwise mutation can locate any solution with a positive probability, the (1+1)EA applied to OneMax generates an elitist search with  $r_{i,j}>0,\forall\,0\le i\le j\le L$. Thus, We estimate its approximation error by constructing an auxiliary bi-diagonal search.

For a solution $\mathbf{x}$ with status $j$, let us consider a special case that only one `0' is flip to `1', which leads to decrease of approximation error and the status transition from $j$ to $j-1$. Noting that it happens with a probability  $\frac{j}{n}\left(1-\frac{1}{n}\right)^{n-1}$, we conclude that
\begin{equation*}
  r_{j-1,j}\ge \frac{j}{n}\left(1-\frac{1}{n}\right)^{n-1},\quad j=1,\dots,n.
\end{equation*}
Then, for the elitist search with probability transition matrix $\mathbf{\tilde{R}}=(r_{i,j})_{i,j=0,\dots,n}$, we construct an auxiliary bi-diagonal search with probability transition matrix
\begin{align*}
  &\mathbf{\tilde{S}}=(s_{i,j})_{i,j=0,\dots,n}\\
  =&\left(
                 \begin{array}{cccccc}
                 1& \frac{1}{n}(1-\frac{1}{n})^{n-1} & & & & \\
                   & 1-\frac{1}{n}(1-\frac{1}{n})^{n-1} & \frac{2}{n}(1-\frac{1}{n})^{n-1} &   &   &   \\
                   &   & 1-\frac{2}{n}(1-\frac{1}{n})^{n-1} & \frac{3}{n}(1-\frac{1}{n})^{n-1} &   &   \\
                   &   &   & \ddots & \ddots &   \\
                   &   &   &   & 1-\frac{n-1}{n}(1-\frac{1}{n})^{n-1} & (1-\frac{1}{n})^{n-1} \\
                   &   &   &   &   & 1-(1-\frac{1}{n})^{n-1} \\
                 \end{array}
               \right).
\end{align*}

It is trivial to verify that $\mathbf{\tilde{R}}$ and $\mathbf{\tilde{S}}$ satisfy conditions (\ref{conC1})-(\ref{conC3}), and the result of Theorem \ref{Th_ES} holds. Then, Theorem \ref{Th_BD} implies that
\begin{equation}\label{err2}
  e^{[t]}\le\mathbf{e}'\mathbf{S}^t\mathbf{p}^{[0]}=\sum_{j=1}^{n}\lambda_j^t\left(\mathbf{e}'\mathbf{p}_j\right)\left(\mathbf{q}_j'\mathbf{p}^{[0]}\right),
\end{equation}
 where
 \begin{equation}\label{lambdaTheo5}
   \lambda_j=1-\frac{j}{n}\left(1-\frac{1}{n}\right)^{n-1},
 \end{equation}
 \begin{align}
  \mathbf{p}_{j}&=\left(\prod_{k=1}^{j-1}\frac{s_{k,k+1}}{s_{j,j}-s_{k,k}},\prod_{k=2}^{j-1}\frac{s_{k,k+1}}{s_{j,j}-s_{k,k}},\dots,\frac{s_{j-1,j}}{s_{j,j}-s_{j-1,j-1}},1,0,\dots,0\right)', \label{EigenTheo5} \\
  \mathbf{q}'_{j}&=\left(0,\dots,0,1,\frac{s_{j,j+1}}{s_{j,j}-s_{j+1,j+1}},\prod_{k=j+1}^{j+2}\frac{s_{k-1,k}}{s_{j,j}-s_{k,k}},\dots,\prod_{k=j+1}^{n}\frac{s_{k-1,k}}{s_{j,j}-s_{k,k}}\right),\label{ColTransMatrixInTheo5}
\end{align}
$j=1\dots,n$. Similar to computation of $\mathbf{p}_{j}$ and $\mathbf{q}'_{j}$ in \ref{AppendixAA}, we know that the values of $\mathbf{p}_{j}$ and $\mathbf{q}_{j}$ defined by (\ref{EigenTheo5}) and (\ref{ColTransMatrixInTheo5}) are also confirmed by (\ref{p1}) and (\ref{q1}), respectively. Submitting (\ref{ErOneMax}), (\ref{DiOneMax}), (\ref{lambdaTheo5}), (\ref{temp1}) and (\ref{temp2}) to (\ref{err2}) we conclude that
  $$e^{[t]}\le\sum_{j=1}^{n}\lambda_j^t\left(\mathbf{e}'\mathbf{p}_j\right)\left(\mathbf{q}'_j\mathbf{p}^{[0]}\right)=\lambda_1^t(\mathbf{e}'\mathbf{p}_1)(\mathbf{q}'_1\mathbf{p}^{[0]})=\frac{n}{2}\left(1-\frac{1}{n}\left(1-\frac{1}{n}\right)^{n-1}\right)^t\le\frac{n}{2}\left(1-\frac{1}{ne}\right)^t.$$
\end{proof}

To demonstrate how tight the estimated upper bound is, we perform comparison between the simulation results and the estimated bound. For the 10-, 20-,...,90-D OneMax problems, the simulated approximation error averaged for 1000 independent runs are compared with the theoretical upper bound presented in Theorem \ref{Th_EA_OneMax}. Just as illustrated in Figure \ref{fig1}, it is showed that the estimated upper bound is very tight for low-dimensional OneMax problems, and the difference between simulation results and the upper bound increases slowly with increase of problem dimension. Similar to the estimation for approximation error of RLS, with the iteration budget $t=an^k$, $(a>0,k\in\mathbb{Z}^+)$,
the asymptotic expected approximation error of (1+1)EA is $C^{-\Omega(n^{k-1})}$ for some constant $C>1$.

\begin{figure}
  \centering
  \includegraphics[width=15cm]{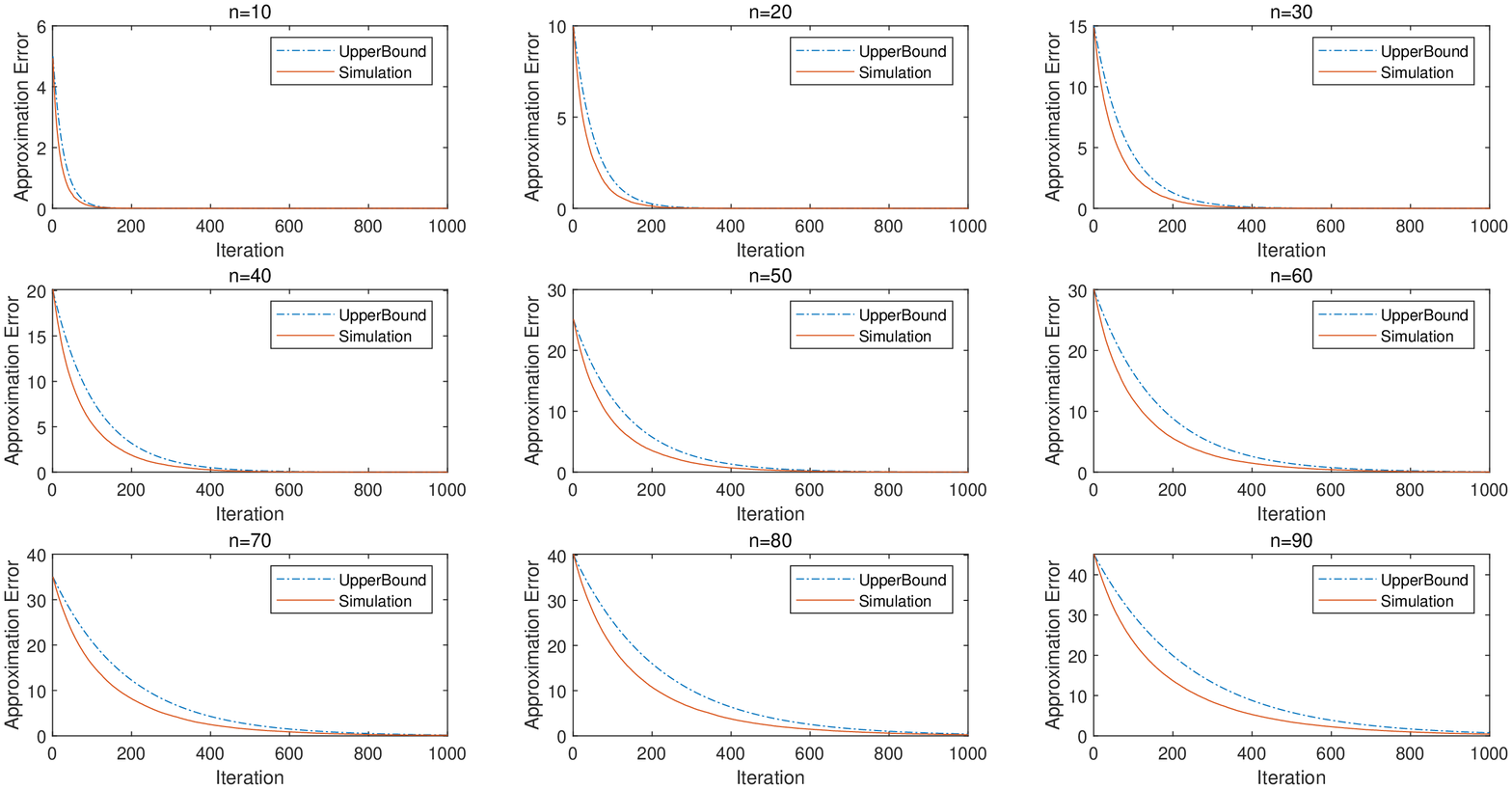}\\
  \caption{Comparison between the estimated upper bound and simulation results on expected approximation error of (1+1)EA solving the 10-, 20-,...,90-D OneMax problems.}\label{fig1}
\end{figure}

\subsection{The Peak Problem}
The global optimal solution of the Peak problem is $\mathbf{x}^*=(1,\dots,1)$, and all other solutions constitute  a platform where all solutions have the identical function value $0$. By defining the status index $i$ as the total amount of 0-bits in a solution $\mathbf{x}$, we know  $\mathbf{\tilde{e}}=(0,1,\dots,1)'$. Correspondingly, $\mathbf{p}^{[0]}=(C_n^0/2^n, C_n^1/2^n,C_n^2/2^n,\dots,C_n^n/2^n)'$.

\begin{theorem}\label{Th_RLS_Peak}
For RLS on the Peak problem,
\begin{equation*}
   e^{[t]}=1-\frac{n+1}{2^n}+\frac{n}{2^n}\left(1 -\frac{1}{n}\right)^t.
\end{equation*}
\end{theorem}
\begin{proof}
When the RLS is employed to solve the Peak problem, the one-bit mutation generate a probability distribution of status transition as
\begin{equation*}
  r_{i,j}=\left\{\begin{aligned}&1/n, && i=0,j=1;\\&1-1/n ,&& i=j=1,\\ & 1;&& i=j\neq 1;\\ & 0,&&\mbox{otherwise}.\end{aligned}\right.
\end{equation*}
Thus, we get the transition submatrix as
\begin{equation*}
  \mathbf{R}=diag\left(1-\frac{1}{n},1,\dots,1\right).
\end{equation*}
Applying Theorem \ref{Th_D} we know that
\begin{equation*}
  e^{[t]}=\mathbf{e}'\mathbf{R}^t\mathbf{p}^{[0]}=\sum_{i=1}^{n}e_ir_{i,i}^tp_i=\left(1 -\frac{1}{n}\right)^t\frac{C_n^1}{2^n}+\sum_{i=2}^n\frac{C_n^i}{2^n}=1-\frac{n+1}{2^n}+\frac{n}{2^n}\left(1 -\frac{1}{n}\right)^t.
\end{equation*}
\end{proof}

{\color{red}Since the error vector of non-optimal status is $\mathbf{e}=(1,\dots,1)'$, the obtained expected approximation error is equal to the probability to stay at non-optimal statuses. Because RLS employs a one-bit mutation, the optimal solution is achievable if and only if the initial solution is located at statuses $0$ and $1$. On the contrary, it cannot jump out of the fitness platform if the initial solution is not located adjacent to the global optimal solution. Thus, the probability to stay at non-optimal statuses would not converge to zero when $t\to\infty$, and its global convergence to the optimal solution cannot be guaranteed.}

 Presentation of this case is to show that the transition submatrix could be diagonal when there is a fitness platform, and so, it is easy to compute the expected approximation error. Fortunately, such an diagonal transition submatrix is also available when the bitwise mutation is employed.

\begin{theorem}\label{Th_EA_Peak}
For (1+1)EA on the Peak problem,
\begin{equation*}
e^{[t]}=\sum_{i=1}^{n}\left[1-\left(\frac{1}{n-1}\right)^i\left(1-\frac{1}{n}\right)^{n}\right]^t\frac{C_n^i}{2^n}.
\end{equation*}
\end{theorem}
\begin{proof}
When the (1+1)EA is employed to solve the Peak problem, the transition probability
\begin{equation*}
  r_{i,j}=\left\{\begin{aligned}&(1/n)^j(1-1/n)^{n-j}, && i=0,j\neq 0;\\&1 ,&& i=j=0,\\ & 1-(1/n)^j(1-1/n)^{n-j};&& i=j\neq 1;\\ & 0,&&\mbox{otherwise}.\end{aligned}\right.
\end{equation*}
Then, it holds that
\begin{equation*}
  \mathbf{R}=diag\left(1-\frac{1}{n}\left(1-\frac{1}{n}\right)^{n-1},1-\left(\frac{1}{n}\right)^2\left(1-\frac{1}{n}\right)^{n-2},\dots,1-\left(\frac{1}{n}\right)^{n-1}\left(1-\frac{1}{n}\right),1-\left(\frac{1}{n}\right)^{n}\right).
\end{equation*}
Applying Theorem \ref{Th_D}, we know
\begin{equation*}
  e^{[t]}=\sum_{i=1}^{n}e_ir_{i,i}^tp_i=\sum_{i=1}^{n}\left[1-\left(\frac{1}{n}\right)^i\left(1-\frac{1}{n}\right)^{n-i}\right]^t\frac{C_n^i}{2^n}=\sum_{i=1}^{n}\left[1-\left(\frac{1}{n-1}\right)^i\left(1-\frac{1}{n}\right)^{n}\right]^t\frac{C_n^i}{2^n}.
\end{equation*}
\end{proof}

\subsection{The Deceptive Problem}
According to definition of the Deceptive problem, the mapping from the total amount of 1-bits to the fitness and approximation error of $\mathbf{x}$ is as follows.
\begin{equation}\label{StaTranDec}
\begin{array}{lccccc}
  |\mathbf{x}|:\quad\quad & 0& 1 & \cdots & n-1 & n \\
  & \downarrow & \downarrow  & \cdots & \downarrow & \downarrow\\
  f(\mathbf{x}):\quad\quad& n-1 & n-2 & \cdots & 0 & n \\
  & \downarrow & \downarrow  & \cdots & \downarrow & \downarrow\\
  e(\mathbf{x}):\quad\quad& 1 & 2 & \cdots & n & 0
\end{array}
\end{equation}
Then, the feasible solution set could be divided into $n+1$ subsets, and there exists a local optimal state with $|\mathbf{x}|=0$. Then, we have
\begin{equation}\label{ErDec}
   \mathbf{\tilde{e}}=(0,\mathbf{e}')'=(0,1,2,\dots,n)',
 \end{equation}
 and it holds
\begin{equation}\label{DiDec} \mathbf{\tilde{p}}^{[0]}=\left(\frac{C_n^n}{2^n},\mathbf{{p}}^{\prime[0]}\right)'=\left(\frac{C_n^n}{2^n},\frac{C_n^0}{2^n},\frac{C_n^1}{2^n},\dots,\frac{C_n^{n-1}}{2^n}\right)'.
 \end{equation}

\begin{theorem}\label{Th_RLS_Deceptive}
For RLS applied to the Deceptive problem,
\begin{equation*}
   e^{[t]}=\left(1-\frac{1}{2^{n-1}}\right)+\left(\frac{n}{2}-\frac{n}{2^{n-1}}\right)\left(1-\frac{1}{n}\right)^t.
\end{equation*}
\end{theorem}

\begin{proof}
If the RLS is employed solving the Deceptive problem, we have
\begin{equation}\label{TrDeRLS}
  \mathbf{R}=(r_{i,j})_{n\times n}\left(
                 \begin{array}{ccccc}
                   1 & 1/n &   &   &   \\
                     & 1-1/n & 2/n &   &   \\
                     &   & \ddots & \ddots &   \\
                     &   &   & 2/n & (n-1)/n \\
                     &   &   &   & 0 \\
                 \end{array}
               \right).
\end{equation}
Then, Theorem \ref{Th_BD} implies that
\begin{equation}\label{err}
  e^{[t]}=\mathbf{e}'\mathbf{R}^t\mathbf{p}^{[0]}=\sum_{j=1}^{n}\lambda_j^t\left(\mathbf{e}'\mathbf{p}_j\right)\left(\mathbf{q}_j'\mathbf{p}^{[0]}\right),
\end{equation}
where $\lambda_j=r_{j,j}$, $\mathbf{p}_j$ and $\mathbf{q}_{j}$ are defined by (\ref{Eigen1}) and (\ref{ColTransMatrixIn}), respectively. From (\ref{ep1}), (\ref{ep2}) and (\ref{qq1}), we conclude that
\begin{align*}
e^{[t]}=\lambda_1^t(\mathbf{e}'\mathbf{p}_1)(\mathbf{q}'_1\mathbf{p}^{[0]})+\lambda_2^t(\mathbf{e}'\mathbf{p}_2)(\mathbf{q}'_2\mathbf{p}^{[0]})=\left(1-\frac{1}{2^{n-1}}\right)+\left(\frac{n}{2}-\frac{n}{2^{n-1}}\right)\left(1-\frac{1}{n}\right)^t.
\end{align*}
\end{proof}

{\color{red} Because the Deceptive problem has a local absorbing region where individuals cannot jump out by the one-bit mutation, the expected approximation error would not converge to zero when $t\to\infty$. Then, the global search strategy, that is, the bitwise mutation, is needed to get the global convergence of RSH. To estimate the expected approximation error of (1+1)EA on the Deceptive problem, we need the results presented in the following lemma.}
\begin{lemma}\label{Sec4L5}
  Consider a Markov chain model of Algorithm \ref{alg1} whose transition matrix can be partitioned as
  \begin{align*}
  \mathbf{\tilde{R}}=\left(
                       \begin{array}{cc}
                         \mathbf{\hat{R}} & \mathbf{\hat{r}}^{[1]}  \\
                         0 & r_{L,L}  \\
                       \end{array}
                     \right).
\end{align*}
Correspondingly, denote
\begin{equation*}
  \mathbf{\tilde{e}}=(\mathbf{\hat{e}}', e_{L})',\quad \mathbf{\tilde{p}}^{[0]}=(\mathbf{\hat{p}}^{[0]'}, p_{L}^{[0]})'.
\end{equation*}
Then, it holds for the expected approximation error that
\begin{align*}
  e^{[t]}=\mathbf{\hat{e}}'\mathbf{\hat{R}}^t\mathbf{\hat{p}}^{[0]}+p_L^{[0]}\sum_{k=0}^{t-1}r_{L,L}^{k}\mathbf{\hat{e}}'\mathbf{\hat{R}}^{t-1-k}\mathbf{\hat{r}}^{[1]}+p_L^{[0]}e_{L}r_{L,L}^t. 
\end{align*}

\end{lemma}
\begin{proof} According to the partition of transition matrix, we know
  \begin{align*}
  e^{[t]}&=\mathbf{\tilde{e}}'\mathbf{\tilde{R}}^t\mathbf{\tilde{p}}^{[0]} =(\mathbf{\hat{e}}',e_{L})\left(
                       \begin{array}{cc}
                         \mathbf{\hat{R}} & \mathbf{\hat{r}}^{[1]}  \\
                         0 & r_{L,L}  \\
                       \end{array}
                     \right)^t(\mathbf{\hat{p}}^{[0]'},p_L^{[0]})'=(\mathbf{\hat{e}}',e_{L})\left(
                       \begin{array}{cc}
                         \mathbf{\hat{R}}^t & \mathbf{\hat{r}}^{[t]}  \\
                         0 & r_{L,L}^t  \\
                       \end{array}
                     \right)(\mathbf{\hat{p}}^{[0]'},p_L^{[0]})',
  \end{align*}
 where $\mathbf{\hat{r}}^{[t]}=\sum_{k=0}^{t-1}r_{L,L}^{k}\mathbf{\hat{R}}^{t-1-k}\mathbf{\hat{r}}^{[1]}$. Thus,
  \begin{align*}
  e^{[t]}&=\mathbf{\hat{e}}'\mathbf{\hat{R}}^t\mathbf{\hat{p}}^{[0]}+p_L^{[0]}(\mathbf{\hat{e}}'\mathbf{\hat{r}}^{[t]}+e_{L}r_{L,L}^t)
  =\mathbf{\hat{e}}'\mathbf{\hat{R}}^t\mathbf{\hat{p}}^{[0]}+p_L^{[0]}\sum_{k=0}^{t-1}r_{L,L}^{k}\mathbf{\hat{e}}'\mathbf{\hat{R}}^{t-1-k}\mathbf{\hat{r}}^{[1]}+p_L^{[0]}e_{L}r_{L,L}^t. 
\end{align*}
\end{proof}

\begin{theorem}\label{Th_EA_Deceptive}
The expected approximation error of (1+1)EA for the Deceptive problem is bounded by
\begin{equation*}
  e^{[t]}\le \left(1-\frac{n+1}{2^n}+\frac{en^2}{2^n}\right)\left[1-\left(\frac{1}{n}\right)^n\right]^t+\left(\frac{n}{2}-\frac{n}{2^n}+\frac{en^2}{2^n}\right)\left[1-\frac{1}{en}\right]^t+\frac{n^2}{2^n}\left[1-\frac{1}{e}\right]^t.
\end{equation*}
\end{theorem}
\begin{proof}
Partition the transition matrix as
\begin{align}
  \mathbf{\tilde{R}}=\left(
                       \begin{array}{cc}
                         \mathbf{\hat{R}} & \mathbf{\hat{r}}^{[1]}  \\
                         0 & r_{n,n}  \\
                       \end{array}
                     \right),
\end{align}
where $\mathbf{\hat{R}}=(r_{i,j})_{i,j=0,1,\dots,n}$, $\mathbf{\hat{r}}^{[1]}=(r_{0,n},r_{1,n},\dots,r_{n-1,n})'$. Denote
\begin{align*}
  &\mathbf{\hat{e}}=(e_0,\dots,e_{n-1})'=(0,\dots,n-1)', \\
  &\mathbf{\hat{p}}^{[0]}=\left(p_0^{[0]},\dots,p_{n-1}^{[0]}\right)'=\left(\frac{C_n^n}{2^n},\frac{C_n^0}{2^n},\frac{C_n^1}{2^n},\dots,\frac{C_n^{n-2}}{2^n}\right)'.
\end{align*}
By Lemma \ref{Sec4L5}, we know
  \begin{align}
  e^{[t]}&=\mathbf{\hat{e}}'\mathbf{\hat{R}}^t\mathbf{\hat{p}}^{[0]}+p_n^{[0]}(\mathbf{\hat{e}}'\mathbf{\hat{r}}^{[t]}+e_{n}r_{n,n}^t)
  =\mathbf{\hat{e}}'\mathbf{\hat{R}}^t\mathbf{\hat{p}}^{[0]}+p_n^{[0]}\sum_{k=0}^{t-1}r_{n,n}^{k}\mathbf{\hat{e}}'\mathbf{\hat{R}}^{t-1-k}\mathbf{\hat{r}}^{[1]}+p_n^{[0]}e_{n}r_{n,n}^t. \label{temp3}
\end{align}

If the (1+1)EA is employed, probability to flip $j$ bits is  $\left(\frac{1}{n}\right)^j\left(1-\frac{1}{n}\right)^{n-j}$, and the probability to flip one of $j$ bits is  $C_{j}^1\left(\frac{1}{n}\right)\left(1-\frac{1}{n}\right)^{n-1}$. Thus, by (\ref{StaTranDec}) we know
\begin{align*}
  & r_{0,j} =\left(\frac{1}{n}\right)^{n+1-j}\left(1-\frac{1}{n}\right)^{j-1},\quad j=1,\dots,n-1,  \\
  & r_{j-1,j} \ge \frac{j-1}{n}\left(1-\frac{1}{n}\right)^{n-1},\quad j=2,\dots,n-1.
\end{align*}
Let
\begin{align}\label{TrOneRLSTheo9}
  \mathbf{\hat{S}}=\left(
                 \begin{array}{ccccc}
                 1&(\frac{1}{n})^n&(\frac{1}{n})^n&\cdots & (\frac{1}{n})^n\\
                      &1-(\frac{1}{n})^n & \frac{1}{n}(1-\frac{1}{n})^{n-1} &   &   \\
                     &   & 1-(\frac{1}{n})^n-\frac{1}{n}(1-\frac{1}{n})^{n-1} & \ddots &   \\
                   &   &   &  \ddots &\\
                   &     &   &   &\frac{n-2}{n}(1-\frac{1}{n})^{n-1}\\
                   &     &   &   & 1-(\frac{1}{n})^n-\frac{n-2}{n}(1-\frac{1}{n})^{n-1} \\
                 \end{array}
               \right).
\end{align}
It is trivial to check that $\mathbf{\hat{R}}$ and $\mathbf{\hat{S}}$ satisfied conditions (\ref{conC1})-(\ref{conC3}). Then, Theorem \ref{Th_ES} implies that
\begin{align}
  &\mathbf{\hat{e}}'\mathbf{\hat{R}}^t\mathbf{\hat{p}}^{[0]}\le \mathbf{\hat{e}}'\mathbf{\hat{S}}^t\mathbf{\hat{p}}^{[0]},\label{temp7}\\
  &\mathbf{\hat{e}}'\mathbf{\hat{R}}^{t-1-k}\mathbf{\hat{r}}^{[1]}\le \mathbf{\hat{e}}'\mathbf{\hat{S}}^{t-1-k}\mathbf{\hat{r}}^{[1]}. \label{temp8}
\end{align}

Furthermore, denote
\begin{align}
& \mathbf{\check{R}}=(r_{i,j})_{i,j=1,\dots,n-1},\nonumber\\
 & \mathbf{\check{e}}=(1,\dots,n-1)', \label{etheo9}\\
 & \mathbf{\check{p}}^{[0]}=\left(\frac{C_n^0}{2^n},\frac{C_n^1}{2^n},\dots,\frac{C_n^{n-2}}{2^n}\right)',\label{qtheo9} \\
 & \mathbf{\check{r}}^{[1]}=(r_{1,n},\dots,r_{n-1,n})', \label{rtheo9}
\end{align}
and let
\begin{align}
  \mathbf{\check{S}}=&(s_{i,j})_{i,j=1,\dots,n-1}\nonumber\\
  =&\left(\begin{array}{cccc}
                      1-(\frac{1}{n})^n & \frac{1}{n}(1-\frac{1}{n})^{n-1} &   &   \\
                        & 1-(\frac{1}{n})^n-\frac{1}{n}(1-\frac{1}{n})^{n-1} & \ddots &   \\
                      &   &  \ddots &\\
                        &   &   &\frac{n-2}{n}(1-\frac{1}{n})^{n-1}\\
                        &   &   & 1-(\frac{1}{n})^n-\frac{n-2}{n}(1-\frac{1}{n})^{n-1} \\
                 \end{array}
               \right).
\end{align}
By equation (\ref{temp3}), we know
\begin{align}
e^{[t]}&=\mathbf{\hat{e}}'\mathbf{\hat{R}}^t\mathbf{\hat{p}}^{[0]}+p_n^{[0]}\sum_{k=0}^{t-1}r_{n,n}^{k}\mathbf{\hat{e}}'\mathbf{\hat{R}}^{t-1-k}\mathbf{\hat{r}}^{[1]}+p_n^{[0]}e_{n}r_{n,n}^t\nonumber\\
  &\le \mathbf{\check{e}}'\mathbf{\check{S}}^t\mathbf{\check{p}}^{[0]}+p_n^{[0]}\sum_{k=0}^{t-1}r_{n,n}^{k}\mathbf{\check{e}}'\mathbf{\check{S}}^{t-1-k}\mathbf{\check{r}}^{[1]}+p_n^{[0]}e_{n}r_{n,n}^t \quad\mbox{(by (\ref{temp7}) and (\ref{temp8}))}\nonumber\\
  &= \mathbf{\check{e}}'\left(\sum_{j=1}^{n-1}\lambda_j^t\mathbf{\check{p}}_j\mathbf{\check{q}}'_j\right)\mathbf{\check{p}}^{[0]}+p_n^{[0]}\sum_{k=0}^{t-1}r_{n,n}^{k}\mathbf{\check{e}}'\left(\sum_{j=1}^{n-1}\lambda_j^{t-1-k}\mathbf{\check{p}}_j\mathbf{\check{q}}'_j\right)\mathbf{\check{r}}^{[1]}+p_n^{[0]}e_{n}r_{n,n}^t\quad\mbox{(by Lemma \ref{Sec2L3})}\nonumber\\
  &= \sum_{j=1}^{n-1}\lambda_j^t(\mathbf{\check{e}}'\mathbf{\check{p}}_j)(\mathbf{\check{q}}'_j\mathbf{\check{p}}^{[0]})+p_n^{[0]}\sum_{k=0}^{t-1}r_{n,n}^{k}\sum_{j=1}^{n-1}\lambda_j^{t-1-k} (\mathbf{\check{e}}'\mathbf{\check{p}}_j)(\mathbf{\check{q}}'_j\mathbf{\check{r}}^{[1]})+p_n^{[0]}e_{n}r_{n,n}^t, \label{etheo9}
\end{align}
 where
 \begin{equation}\label{lambdaTheo9}
   \lambda_j=s_{j,j}=1-\left(\frac{1}{n}\right)^n-\frac{j-1}{n}\left(1-\frac{1}{n}\right)^{n-1},
 \end{equation}
 \begin{align}
  \mathbf{p}_{j}&=\left(\prod_{k=1}^{j-1}\frac{s_{k,k+1}}{s_{j,j}-s_{k,k}},\prod_{k=2}^{j-1}\frac{s_{k,k+1}}{s_{j,j}-s_{k,k}},\dots,\frac{s_{j-1,j}}{s_{j,j}-s_{j-1,j-1}},1,0,\dots,0\right)', \label{EigenTheo9} \\
  \mathbf{q}'_{j}&=\left(0,\dots,0,1,\frac{s_{j,j+1}}{s_{j,j}-s_{j+1,j+1}},\prod_{k=j+1}^{j+2}\frac{s_{k-1,k}}{s_{j,j}-s_{k,k}},\dots,\prod_{k=j+1}^{n-1}\frac{s_{k-1,k}}{s_{j,j}-s_{k,k}}\right),\label{ColTransMatrixInTheo9}
\end{align}
$j=1\dots,n-1$.
Substituting (\ref{lambdaTheo9}), (\ref{ep22}), (\ref{qq2}) and (\ref{qr1}) to (\ref{etheo9}) we conclude that
\begin{align}
e^{[t]}&\le \sum_{j=1}^{2}\lambda_j^t(\mathbf{\check{e}}'\mathbf{\check{p}}_j)(\mathbf{\check{q}}'_j\mathbf{\check{p}}^{[0]})+p_n^{[0]}\sum_{j=1}^{2}\sum_{k=0}^{t-1}r_{n,n}^{k}\lambda_j^{t-1-k} (\mathbf{\check{e}}'\mathbf{\check{p}}_j)(\mathbf{\check{q}}'_j\mathbf{\check{r}}^{[1]})+p_n^{[0]}e_{n}r_{n,n}^t\nonumber\\
&\le \sum_{j=1}^{2}\lambda_j^t(\mathbf{\check{e}}'\mathbf{\check{p}}_j)(\mathbf{\check{q}}'_j\mathbf{\check{p}}^{[0]})+p_n^{[0]} \left[\left(1-\frac{n+1}{n}\left(1-\frac{1}{n}\right)^{n-1}\right)\left(\frac{\lambda_1^{t}}{\frac{1}{n}(1-\frac{1}{n})^{n-1}}+ \frac{\lambda_2^{t}}{\frac{1}{n}(1-\frac{1}{n})^{n-1}}\right)+e_{n}r_{n,n}^t\right]\nonumber\\
&\le\left(1-\frac{n+1}{2^n}\right)\lambda_1^t+\left(\frac{n}{2}-\frac{n}{2^n}\right)\lambda_2^t+\frac{en^2}{2^n}(\lambda_1^t+\lambda_2^t)+\frac{n^2}{2^n}r_{n,n}^t\nonumber\\
&\le \left(1-\frac{n+1}{2^n}+\frac{en^2}{2^n}\right)\left[1-\left(\frac{1}{n}\right)^n\right]^t+\left(\frac{n}{2}-\frac{n}{2^n}+\frac{en^2}{2^n}\right)\left[1-\frac{1}{en}\right]^t+\frac{n^2}{2^n}\left[1-\frac{1}{e}\right]^t.\nonumber
%
\end{align}

\end{proof}

The estimated upper bound is again evaluated by comparing it with simulation results. For the 10-, 20-,...,90-D deceptive problems, the simulated approximation errors are averaged for 1000 independent runs. Just as illustrated in Figure \ref{fig2}, it is showed that the estimated upper bound is very tight for the deceptive problems, which demonstrates that the error analysis method generates a tight upper bound for the approximation error of (1+1)EA solving the deceptive problem.

{\color{red} Since the deceptive problem has a local optimal solution adjacent to the global optimal solution, it is difficult for the (1+1)EA to jump from the local optimum to the global one. So, the expected approximation error is dominated by the first item  $\left(1-\frac{n+1}{2^n}+\frac{en^2}{2^n}\right)[1-\left(\frac{1}{n}\right)^n]^t$, and a iteration budget of order $\Theta(n^n)$ is necessary to get satisfactory convergence performance.}
\begin{figure}
  \centering
  \includegraphics[width=15cm]{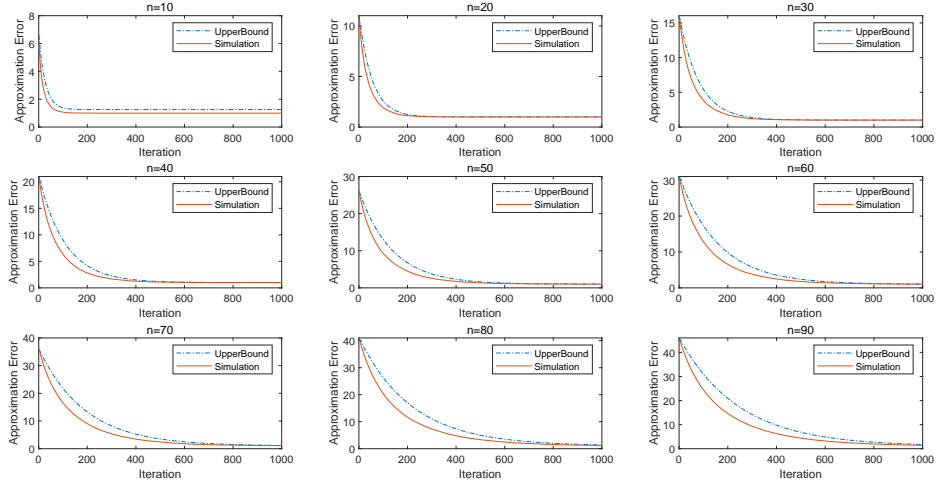}\\
  \caption{Comparison between the estimated upper bound and simulation results on expected approximation error of (1+1)EA solving the 10-, 20-,...,90-D Deceptive problems.}\label{fig2}
\end{figure}

{
\section{The Knapsack Problem}\label{SecKnap}
To validate power of error analysis, we further investigate an instance of the knapsack problem~\cite{he2014knapsack},
\begin{equation}\label{P4}
  \begin{aligned}\max &\quad f(\mathbf x)= \sum_{i=1}^np_ix_i,  \\
  s.t. &\quad \sum_{i=1}^nw_ix_i\le W,\,x_{i}\in\{0,1\},i=1,\dots,n.\end{aligned}
\end{equation}
where the problem parameters is presented in Tab. \ref{Knap}. Without loss of generality, suppose that $\alpha$ is a constant in $(0,1)$, and $\alpha n$ is a positive integer for sufficiently large $n$.


Global optimal solution of the investigated knapsack problem is $\mathbf x_g^{*}=(1,0,\dots,0)$ with $f(\mathbf x_g^*)=n$, and  $\mathbf x_l^*=(0,1,\dots,1,0,\dots 0)$ is the local optimal solution with $f(\mathbf x_l^*)=\alpha n-1$. Since the penalty method would introduce an penalty parameter that function on the fitness value of solutions, in error analysis, we employ a ratio-greedy repair mechanism to transform infeasible solutions into feasible ones. That is, if  an infeasible solution is generated, we sort all items according to the profit-to-weight(P-W) ratios, and the items with the smallest P-W ratio are successively removed from the knapsack until a feasible solution is achieved.

Denote a solution of the knapsack problem as $\mathbf{x}=(x_1,\dots,x_n)$, where $x_i\in\{0,1\}$, $i=1,\dots,n$. If item $i$ is put into the knapsack, then the corresponding binary variable $x_i$ is set as  `1'. According to the P-W ratios of items(variables), we can separate the solution vector $\mathbf x$  into three sub-vectors.
\begin{itemize}
  \item $\mathbf{x}_1=(x_1)$, the variable in which corresponds to the first item with the P-W ratio $1$. The solution $\mathbf{x}_g^*=(\mathbf x_1,\mathbf x_2,\mathbf x_3)=(1,\mathbf 0,\mathbf 0)$ represents the best packing solution that only contains the first item.
  \item $\mathbf{x}_2=(x_2,\dots,x_{\alpha n})$,  variables in which correspond to the $2^{th}-\alpha n^{th}$ items with the P-W ratio $\alpha n$. Since weights of these $\alpha n-1$ items are $\frac{1}{\alpha n}$, the total weight of a solution including all of them is $1-\frac{1}{\alpha n}$, and the total profit is $\alpha n-1$. That is, $\mathbf x_l^*=(\mathbf{x}_1,\mathbf{x}_2,\mathbf{x}_3)=(0,\mathbf{1},\mathbf{0})$ is a local optimal solution with $f(\mathbf x_l^*)=\alpha n-1$. $\mathbf{x}_2=(x_2,\dots,x_{\alpha n})$.
  \item $\mathbf{x}_3=(x_{\alpha n+1},\dots,x_{n})$, variables in which correspond to the last $n-\alpha n$ items with the P-W ratio $\frac{1}{n^2}$.
\end{itemize}
\begin{table}
\caption{Parameters of the knapsack problem.}
\label{Knap}
\centering
\begin{tabular}{c|c|c|c}
  \hline\hline
  \multirow{2}{*}{$\mathbf{x}$}  & $\mathbf{x}_1$ &
  $\mathbf{x}_2$ & $\mathbf{x}_3$ \\
  \cline{2-4}
    & $x_1$ & $x_2,\dots,x_{\alpha n}$ & $x_{\alpha n+1},\dots,x_{n}$ \\
  \hline

  Item $i$ & $1$ & $2,\dots,\alpha n$ & $\alpha n+1,\dots,n$ \\
  \hline

  Profit $p_i$ & $n$ & $1$ & $\frac{1}{n}$ \\
  Weight $w_i$ & $n$ & $\frac{1}{\alpha n}$ & $n$ \\
  P-W ratio $\frac{p_i}{w_i}$ & $1$ & $\alpha n$ & $\frac{1}{n^2}$ \\
  \hline
  Capacity $W$ & \multicolumn{3}{c}{$n$}\\
  \hline\hline
\end{tabular}
\end{table}

Because all items corresponding to $\mathbf x_1$ and $\mathbf x_3$ have a weight $n$, a solution $\mathbf x=(\mathbf x_1,\mathbf x_2,\mathbf x_3)$ is feasible if and only if at most one of $\mathbf{x}_1$, $\mathbf{x}_2$ and $\mathbf{x}_3$ is non-zero.
\begin{itemize}
  \item  If $|\mathbf x_1|=|\mathbf x_2|=|\mathbf x_3|=0$, $\mathbf{x}$ represents an empty knapsack with a fitness $0$;
  \item  if $|\mathbf x_1|=1$, $|\mathbf x_2|=|\mathbf x_3|=0$, we get the global optimal solution $\mathbf{x}^*_g$;
  \item if $|\mathbf x_2|\ge 1$, $|\mathbf x_1|=|\mathbf x_3|=0$, fitness of $\mathbf{x}$ is $|\mathbf x_2|$.
  \item if $|\mathbf x_3|= 1$, $|\mathbf x_1|=|\mathbf x_2|=0$, we have $f(\mathbf x)=\frac{1}{n}$.
\end{itemize}

When a solution $\mathbf x=(\mathbf x_1,\mathbf x_2,\mathbf x_3)$ contains more than one non-zero sub-vectors, it is infeasible, and the ratio-greedy strategy would be triggered to generate a feasible one.
According to the value of P-W ratio, the repair strategy would first remove items represented by $\mathbf{x}_3$, then flip variables in $\mathbf{x}_1$ to zero. Moreover, an infeasible solution could be form of $\mathbf{x}=(0,\mathbf{0},\mathbf{x}_3)$ when $|\mathbf{x}_3|\ge 2$. For this case, the repair strategy would randomly delete redundant items until only one variable in $\mathbf{x}_3$ is `1'.

%
%
%

\subsection{The Error Vector of Feasible Solutions}
Possible fitness values of feasible solutions are $n,\alpha n-1,\dots,1,\frac{1}{n},0 $. According to the fitness values of feasible solutions, their statuses are labelled as $s_0,s_1,\dots,s_{\alpha n+1}$. Then, we can get the fitness vector
\begin{equation*}
\mathbf{\tilde{f}}=(f_0,f_1,\dots,f_{\alpha n+1})'=\left(n,\alpha n-1,\dots,1,\frac{1}{n},0\right)'.
\end{equation*}
Correspondingly, the approximation errors of feasible solutions can be represented by
\begin{equation*}
\mathbf{\tilde{e}}=(e_0,\mathbf{e})=(e_0,\dots,e_{\alpha n+1})'=\left(0, n-(\alpha n-1),\dots, n-1,n-\frac{1}{n},n\right)',
\end{equation*}
where
\begin{equation}\label{ErrorKnap}
\mathbf{e}=(e_1,\dots,e_{\alpha n+1})'=\left(n-(\alpha n-1),\dots, n-1,n-\frac{1}{n},n\right)'.
\end{equation}

\subsection{Random Initialization and Initial Probability Distribution}
When generating the initial solution  by random initialization, we get the initial distribution of statuses.
\begin{enumerate}
     \item  If $\mathbf x_1=1$ and $|\mathbf x_2|=0$, the finally obtained feasible solution is the global solution $\mathbf x^*=(1,0,\dots,0)$. Then, the approximation error is $0$, and the corresponding probability $p_0^{[0]}=\frac{1}{2^{\alpha n}}$.

  \item If $\alpha n-1\ge |\mathbf x_2|\ge 1$, it generates feasible solutions  $\mathbf{x}=(\mathbf{0,\mathbf {x}_2,\mathbf{0}})$, no matter what the sub-vectors $\mathbf x_1$ and $\mathbf x_3$ are.  For this case, a feasible solution $\mathbf x$ with $|\mathbf x_2|=i$ is generated with probability $C_{\alpha n-1}^{i}(\frac{1}{2})^{i}(\frac{1}{2})^{\alpha n-1-i}$, $i=\alpha n-1,\dots, 1$. Thus, we get the sub-vector of approximation error
      $$\mathbf{e}_1=(n-(\alpha n-1),n-(\alpha n-2),\dots, n-1)'.$$
      The corresponding sub-vector of probability is denoted by
      $$\mathbf{p}^{[0]}_1=(p_1^{[0]},p_2^{[0]},\dots,p_{\alpha n-1}^{[0]})'=\left(C_{\alpha n-1}^{\alpha n-1}\left(\frac{1}{2}\right)^{\alpha n-1},C_{\alpha n-1}^{\alpha n-2}\left(\frac{1}{2}\right)^{\alpha n-1},\dots,C_{\alpha n-1}^{1}\left(\frac{1}{2}\right)^{\alpha n-1}\right)'.$$

  \item If $|\mathbf x_1|=0$, $|\mathbf x_2|=0$ and $|\mathbf x_3|\ge 1$, only the third category of items are put into the knapsack. Then, the ratio-greedy strategy would randomly delete redundant items until only one is remained. Consequently, the approximation error is $n-\frac{1}{n}$, and we get the corresponding probability
      $$p_{\alpha n}^{[0]}=\sum_{i=1}^{n-\alpha n}C_{n-\alpha n}^{i}\left(\frac{1}{2}\right)^{i}\left(\frac{1}{2}\right)^{n-i}=\left(\frac{1}{2}\right)^{\alpha n}-\frac{1}{2^n}.$$

      \item If $\mathbf x=(\mathbf x_1,\mathbf x_2,\mathbf x_3)=(0,\dots,0)$, the approximation error is $n$, and we get the corresponding probability ${p}^{[0]}_{\alpha n+1}=\frac{1}{2^n}$.
\end{enumerate}

In conclusion, we get the initial probability distribution of statuses:
$$\mathbf{\tilde{p}}^{[0]}=(p_0^{[0]},\mathbf{p}_1^{[0]'},p_{\alpha n+1}^{[0]})'=\left(\frac{1}{2^{\alpha n}}, C_{\alpha n-1}^{\alpha n-1}\left(\frac{1}{2}\right)^{\alpha n-1},\dots,C_{\alpha n-1}^{1}\left(\frac{1}{2}\right)^{\alpha n-1},\left(\frac{1}{2}\right)^{\alpha n}-\frac{1}{2^n},\frac{1}{2^n}\right)'.$$
For non-optimal statuses,
\begin{equation}\label{DisKnap}
  \mathbf{{p}}^{[0]}=(\mathbf p_1^{[0]'},p_{\alpha n+1}^{[0]})'=\left(C_{\alpha n-1}^{\alpha n-1}\left(\frac{1}{2}\right)^{\alpha n-1},\dots,C_{\alpha n-1}^{1}\left(\frac{1}{2}\right)^{\alpha n-1},\left(\frac{1}{2}\right)^{\alpha n}-\frac{1}{2^n},\frac{1}{2^n}\right)'.
\end{equation}

\subsection{Expected Approximation Error of RSHs}

\subsubsection{Expected Approximation Error of RLS}
Assisted by the ratio-greedy repair strategy, RLS generates possible status transitions detailed as follows.
 \begin{enumerate}
   \item While the present status is $s_{\alpha n+1}$, the corresponding solution is $\mathbf{x}=(0,\dots,0)$. Then, any flip from `0' to `1'  generates a solution $\mathbf{y}$ with better fitness, which would be accepted by the elitist selection. The status transitions and corresponding transition probabilities $p_{\mathbf x\to \mathbf y}$ are detailed in Tab. \ref{ProbTran} as \textbf{Case 1}.
   \item While the present status is $s_{\alpha n}$, the corresponding solution is $\mathbf{x}=(0,\mathbf 0,\mathbf x_3)$, where $|\mathbf x_3|=1$. Then, any flip of variables in $\mathbf{x}_3$ from `0' to `1'  generates an infeasible solution because weight of items represented by $\mathbf x_3$ is $n$. Then, the greedy-repair strategy would convert it to another solution at status $s_{\alpha n}$, and the status is not changed. If one bit in $\mathbf x_1$ or $\mathbf x_2$ is flipped to $1$, the repair strategy will keep it and flip the `1' in $\mathbf x_3$ to `0', which results in the status transition labeled  in Tab. \ref{ProbTran} as \textbf{Case 2}.
   \item While the present status is $s_{k}$, $k=\alpha n-1,\dots, 2$, the corresponding solution is $\mathbf{x}=(0,\mathbf x_2,\mathbf 0)$, where $|\mathbf x_2|=\alpha n-k$. Then, a candidate solution is accepted if and only if it is generated by flip another `0' in $\mathbf x_2$  to `1'. The status transition characterized as \textbf{Case 3} in Tab. \ref{ProbTran}.
   \item If the present status is $s_1$, the one-bit mutation cannot generate a status transition any more, and the iteration process would stagnate.
 \end{enumerate}

\begin{table}
\caption{Status transitions and the corresponding probabilities generated by RLS.}\label{ProbTran}
\centering
\begin{tabular}{c|c|ccc|cc|c}
  \hline
  \hline
   & &\multicolumn{3}{c|}{\textbf{Case 1}}& \multicolumn{2}{c|}{\textbf{Case 2}}&  \textbf{Case 3}\\
  \hline

  \multirow{3}{*}{$\mathbf x$}& Status & \multicolumn{3}{c|}{$\alpha n+1$} &\multicolumn{2}{c|}{$\alpha n$}& $ k,(k=2,\dots,\alpha n-1)$\\
  &$f(\mathbf x)$ & \multicolumn{3}{c|}{$0$} &\multicolumn{2}{c|}{$\frac{1}{n}$}& $\alpha n-k,(k=\alpha n-1,\dots,2)$\\
  & $e(\mathbf x)$& \multicolumn{3}{c|}{$n$}  &\multicolumn{2}{c|}{$n-\frac{1}{n}$} & $(1-\alpha)n+k$\\
  \hline

  \multirow{2}{*}{$\mathbf y$}& Status&$\alpha n$ &$\alpha n-1$ & $0$ & $\alpha n-1$ & $0$ & $k-1$\\
  & $f(\mathbf y)$&$\frac{1}{n}$ &$1$ & $n$ & $1$ & $n$ & $\alpha n-k+1$\\
  & $e(\mathbf y)$&$n-\frac{1}{n}$ &$n-1$ & $0$ & $n-1$ & $0$ & $(1-\alpha)n+k-1$\\

  \hline
  \multicolumn{2}{c|}{$p_{\mathbf x\to\mathbf y}$}& $\frac{n-\alpha n}{n}$ & $\frac{\alpha n-1}{n}$&$\frac{1}{n}$ & $\frac{\alpha n-1}{n}$ & $\frac{1}{n}$ &$\frac{k-1}{n}$\\
  \hline
\end{tabular}
\end{table}

Then, the transition matrix can be represented as
\begin{equation}\label{TrKnapRLS}
  \mathbf {\tilde R}=\left(
                       \begin{array}{cc}
                         \mathbf{\hat{R}} & \mathbf{\hat{r}}^{[1]} \\
                         \mathbf{0} & 0 \\
                       \end{array}
                     \right),
\end{equation}
where
\begin{equation}\label{rKnapRLS}
  \mathbf{\hat{r}}^{[1]}=(r_{0,\alpha n+1},\dots,r_{n,\alpha n+1})'=(1/n,0,\dots,0,\alpha-1/n,1-\alpha)',
\end{equation}
\begin{align}\label{TrKnapRLS1}
  \mathbf{\hat{R}}=\left(\begin{array}{cc}
                           1 & \mathbf{\check{r}} \\
                           0 & \mathbf{\check{R}}
                         \end{array}\right)=\left(
                 \begin{array}{ccccccc}
                 1&0& 0 &\cdots & 0 & 0& \frac{1}{n}\\
                      &1 & \frac{1}{n} & \cdots  & 0 & 0 & 0  \\
                     &   & 1-\frac{1}{n} & \ddots &  \vdots & \vdots & \vdots\\
                   &   &   &  \ddots & \vdots &\vdots &\vdots\\
                   &     &   &   &\alpha-\frac{3}{n} & 0 &0 \\
                   &     &   &   & 1-(\alpha-\frac{3}{n}) & \alpha-\frac{2}{n} & 0\\
                     &     &   &   &   & 1-(\alpha-\frac{2}{n}) & \alpha-\frac{1}{n}\\
                       &     &   &   &   &   & 1-\alpha\\
                 \end{array}
               \right).
\end{align}

\begin{theorem}\label{Th_RLS_Knapsack}
For RLS on the Knapsack problem, the expected approximation error is bounded by
\begin{align*}
   &(n-\alpha n+1)\left[1-\frac{1}{2^{\alpha n-1}}-\frac{1}{2^n}\right]+\frac{\alpha n-2}{4}\left(1-\frac{1}{n}\right)^{t-1} +(-1)^{\alpha n}(n-1)\frac{\alpha n+1}{\alpha n}\frac{1}{2^{\alpha n}}(1-\alpha)^t \nonumber\\
   &\le e^{[t]}
   \le (n-\alpha n+1)\left[1-\frac{1}{2^{\alpha n}}+\frac{1}{2^n}\right]+\frac{\alpha n-1}{2}\left(1-\frac{1}{n}\right)^{t-1}+(-1)^{\alpha n}(n-1)\frac{\alpha n+1}{\alpha n}\frac{1}{2^{\alpha n}}(1-\alpha)^t.
  \end{align*}
\end{theorem}
\begin{proof}
By Lemma \ref{Sec4L5}, we know
\begin{align*}
  e^{[t]}=\mathbf{\hat{e}}'\mathbf{\hat{R}}^t\mathbf{\hat{p}}^{[0]}+p_{\alpha n+1}^{[0]}\mathbf{\hat{e}}'\mathbf{\hat{R}}^{t-1}\mathbf{\hat{r}}^{[1]}, 
\end{align*}
where $\mathbf{\hat{R}}$ given by equation (\ref{TrKnapRLS1}),
\begin{align}
  &\mathbf{\hat{e}}=(e_0,\dots,e_{\alpha n})'=\left(0,n-(\alpha n-1),\dots,n-1,n-\frac{1}{n}\right)', \\
  &\mathbf{\hat{p}}^{[0]}=(p_0^{[0]},\dots,p_{\alpha n}^{[0]})'=\left(\frac{1}{2^{\alpha n}}, C_{\alpha n-1}^{\alpha n-1}\left(\frac{1}{2}\right)^{\alpha n-1},\dots,C_{\alpha n-1}^{1}\left(\frac{1}{2}\right)^{\alpha n-1},\left(\frac{1}{2}\right)^{\alpha n}-\frac{1}{2^n}\right)',
\end{align}
Then, Lemma \ref{Sec2L3} and Theorem \ref{Th_BD} imply that
\begin{align}
  e^{[t]}&=\mathbf{\check{e}}'\left(\sum_{j=1}^{\alpha n}\lambda_j^t\mathbf{\check{p}}_j\mathbf{\check{q}}_j'\right)\mathbf{\check{p}}^{[0]} +p_{\alpha n+1}^{[0]}\mathbf{\check{e}}'\left(\sum_{j=1}^{\alpha n}\lambda_j^{t-1}\mathbf{\check{p}}_j\mathbf{\check{q}}_j'\right)\mathbf{\check{r}}^{[0]}\nonumber\\
  &=\sum_{j=1}^{\alpha n}\lambda_j^t\left(\mathbf{\check{e}}'\mathbf{\check{p}}_j\right)\left(\mathbf{\check{q}}_j'\mathbf{\check{p}}^{[0]}\right) +p_{\alpha n+1}^{[0]}\sum_{j=1}^{\alpha n}\lambda_j^{t-1}\left(\mathbf{\check{e}}'\mathbf{\check{p}}_j\right)\left(\mathbf{\check{q}}_j'\mathbf{\check{r}}^{[1]}\right)\label{errTh10},
\end{align}
where \begin{equation}\label{EigTh10}
        \lambda_j=r_{j,j}=\left\{\begin{aligned}& 1-\frac{j-1}{n}, && j=1,\dots,\alpha n-1,\\ & 1-\frac{j}{n}, && j=\alpha n,\end{aligned}\right.
      \end{equation},
\begin{equation}\label{eTh10}
\mathbf{\check{e}}=(e_1,\dots,e_{\alpha n})'=\left( n-(\alpha n-1),\dots, n-1,n-\frac{1}{n}\right)',
\end{equation}
\begin{equation}\label{qTh10}
  \mathbf{\check{p}}^{[0]}=({p}_1^{[0]'},\dots,p_{\alpha n}^{[0]})'=\left(C_{\alpha n-1}^{\alpha n-1}\left(\frac{1}{2}\right)^{\alpha n-1},\dots,C_{\alpha n-1}^{1}\left(\frac{1}{2}\right)^{\alpha n-1},\left(\frac{1}{2}\right)^{\alpha n}-\frac{1}{2^n}\right)',
\end{equation}
\begin{equation}\label{rrKnapRLS}
  \mathbf{\check{r}}^{[1]}=(r_{1,\alpha n+1},\dots,r_{n,\alpha n+1})'=(0,\dots,0,\alpha-1/n,1-\alpha)',
\end{equation}
\begin{align}
  \mathbf{p}_{j}&=\left(\prod_{k=1}^{j-1}\frac{r_{k,k+1}}{r_{j,j}-r_{k,k}},\prod_{k=2}^{j-1}\frac{r_{k,k+1}}{r_{j,j}-r_{k,k}},\dots,\frac{r_{j-1,j}}{r_{j,j}-r_{j-1,j-1}},1,0,\dots,0\right)', \label{Eigen1Th10} \\
  \mathbf{q}'_{j}&=\left(0,\dots,0,1,\frac{r_{j,j+1}}{r_{j,j}-r_{j+1,j+1}},\prod_{k=j+1}^{j+2}\frac{r_{k-1,k}}{r_{j,j}-r_{k,k}},\dots,\prod_{k=j+1}^{\alpha n}\frac{r_{k-1,k}}{r_{j,j}-r_{k,k}}\right),\label{ColTransMatrixInTh10}\\
  &\quad\quad\quad\quad\quad\quad\quad\quad\quad\quad\quad\quad\quad\quad\quad\quad\quad\quad\quad\quad\quad\quad\quad j=1\dots,\alpha n.\nonumber
\end{align}
Substituting (\ref{EigTh10}), (\ref{ep1Appendd}), (\ref{ep2Appendd}), (\ref{qq1Appendd}), (\ref{qq2Appendd}), (\ref{qr1Th10}) and (\ref{qr2Th10}) to (\ref{errTh10}), we know that
\begin{align}
  e^{[t]}=&\lambda_1^{t-1}\left(\mathbf{\check{e}}'\mathbf{\check{p}}_1\right)\left(\lambda_1\mathbf{\check{q}}_1'\mathbf{\check{p}}^{[0]}+p_{\alpha n+1}^{[0]}\mathbf{\check{q}}_1'\mathbf{\check{r}}^{[1]}\right) +\lambda_2^{t-1}\left(\mathbf{\check{e}}'\mathbf{\check{p}}_2\right)\left(\lambda_2\mathbf{\check{q}}_2'\mathbf{\check{p}}^{[0]}+p_{\alpha n+1}^{[0]}\mathbf{\check{q}}_2'\mathbf{\check{r}}^{[1]}\right) \nonumber\\
  &+\lambda_{\alpha n}^{t-1}\left(\mathbf{\check{e}}'\mathbf{\check{p}}_{\alpha n}\right)\left(\lambda_{\alpha n}\mathbf{\check{q}}_{\alpha n}'\mathbf{\check{p}}^{[0]}+p_{\alpha n+1}^{[0]}\mathbf{\check{q}}_{\alpha n}'\mathbf{\check{r}}^{[1]}\right),\label{EERLS}
  \end{align}
  which is bounded by
  \begin{align*}
   &(n-\alpha n+1)\left[1-\frac{1}{2^{\alpha n-1}}-\frac{1}{2^n}\right]+\frac{\alpha n-2}{4}\left(1-\frac{1}{n}\right)^{t-1} +(-1)^{\alpha n}(n-1)\frac{\alpha n+1}{\alpha n}\frac{1}{2^{\alpha n}}(1-\alpha)^t \nonumber\\
   &\le e^{[t]}
   \le (n-\alpha n+1)\left[1-\frac{1}{2^{\alpha n}}+\frac{1}{2^n}\right]+\frac{\alpha n-1}{2}\left(1-\frac{1}{n}\right)^{t-1}(-1)^{\alpha n}(n-1)\frac{\alpha n+1}{\alpha n}\frac{1}{2^{\alpha n}}(1-\alpha)^t.
  \end{align*}
\end{proof}

Similar to the case of the Deceptive problem, the Knapsack problem has a local absorbing region where individuals cannot jump out by the one-bit mutation, and thus, the expected approximation error would not converge to zero when $t\to\infty$. Because $\lambda_1=1$, by formula (\ref{EERLS}) we know that $e^{[t]}$ converges to 0 if and only if  $\left(\mathbf{\check{e}}'\mathbf{\check{p}}_1\right)\left(\lambda_1\mathbf{\check{q}}_1'\mathbf{\check{p}}^{[0]}+p_{\alpha n+1}^{[0]}\mathbf{\check{q}}_1'\mathbf{\check{r}}^{[1]}\right)=0$, which is further equivalent to the statement that both $\mathbf{\check{p}}^{[0]}$ and $p_{\alpha n+1}^{[0]}$ are zero. That is, we must generate the global optimal solution by initialization, which is impossible on the premise that we do not know the exact global optimal solution. In conclusion, $e^{[t]}$ cannot converge to $0$, no matter what initialization strategy is employed by RLS.

\subsubsection{Expected Approximation Error of (1+1)EA}
Denote the transition probability from status $j$ to status $i$ by $\tilde{p}_{i,j}$.  When the bitwise mutation is employed in the (1+1)EA, the transition probabilities are estimated as follows.
\begin{enumerate}
  \item While status $j$ transitions to status $0$, $j=1,\dots,\alpha n+1$, the transition probability is
\begin{equation}\label{TranKnap1}
  \tilde{p}_{0,j}=\left\{\begin{aligned}& \left(\frac{1}{n}\right)^{\alpha n+1-j}\left(1-\frac{1}{n}\right)^{j-1},&&j=1,\dots,\alpha n-1,\\& \frac{1}{n}\left(1-\frac{1}{n}\right)^{\alpha n-1}, && j=\alpha n,\\ & \frac{1}{n}\left(1-\frac{1}{n}\right)^{\alpha n-1}, &&,j=\alpha n+1.\end{aligned}\right.
\end{equation}

\item The probability to transition from status $j$ to status $i (1\le i<j)$ is
\begin{equation}\label{TranKnap2}
  \tilde{p}_{i,j}\ge\left\{\begin{aligned}& C_{j-1}^{j-i}\left(\frac{1}{n}\right)^{j-i}\left(1-\frac{1}{n}\right)^{(\alpha n-1)-(j-i)},&&j=1,\dots,\alpha n-1,\\& C_{\alpha n-1}^{\alpha n-i}\left(\frac{1}{n}\right)^{\alpha n-i}\left(1-\frac{1}{n}\right)^{i-1}, && j=\alpha n,\alpha n+1, 1\le i\le\alpha n-1,\\ & \left(1-\frac{1}{n}\right)^{\alpha n}-\left(1-\frac{1}{n}\right)^{n}, &&,j=\alpha n+1,i=\alpha n.\end{aligned}\right.
\end{equation}

\end{enumerate}

Note that $\tilde{p}_{i,\alpha n}=\tilde{p}_{i,\alpha n+1}$, $i=0,1,\dots,\alpha n-1$, and the difference between $e_{\alpha n}$ and $e_{\alpha n+1}$ is $\frac{1}{n}$, which is an infinitesimal that could be ignored. Then, we redefine individual statuses by combing the status $s_{\alpha n}$ and $s_{\alpha n+1}$ together, and there are $\alpha n+1$ statuses labelled as $s_0,s_1,\dots,s_{\alpha n}$. The corresponding error vector is estimated as
\begin{equation*}
\mathbf{\tilde{e}}_R=(e_0,\dots,e_{\alpha n-1},e_{\alpha n+1})'=\left(0, n-(\alpha n-1),\dots, n-1,n\right)',
\end{equation*}
and the initial distribution vector is
\begin{equation*}\mathbf{\tilde{p}}^{[0]}_R=(p^{[0]}_0,\dots,p^{[0]}_{\alpha n-1},p^{[0]}_{\alpha n}+p^{[0]}_{\alpha n+1})'=\left(\frac{1}{2^{\alpha n}}, C_{\alpha n-1}^{\alpha n-1}\left(\frac{1}{2}\right)^{\alpha n-1},\dots,C_{\alpha n-1}^{1}\left(\frac{1}{2}\right)^{\alpha n-1},\left(\frac{1}{2}\right)^{\alpha n}\right)'.
\end{equation*}
Let $\mathbf{\tilde{R}}=(r_{i,j})_{i,j=0,1,\dots,\alpha n}$ denote the transition matrix for the redefined individual statuses. When the present status is $s_{\alpha n}$, the status would keep unchanged if no better solutions are generated by the bitwise mutation, that is, the first $\alpha n$ bits are not flipped from `0' to `1'. So, we have $r_{\alpha n,\alpha n}=\left(1-\frac{1}{n}\right)^{\alpha n}$. Because the approximation error is magnified when combining two statuses together, the expected approximation error is amplified, too. That is,
{\color{blue}
\begin{equation}\label{errEAKnap}
  e^{[t]}=\mathbf{\tilde{e}}'\mathbf{\tilde{P}}^t\mathbf{\tilde{p}}^{[0]}\le \mathbf{\tilde{e}}_R'\mathbf{\tilde{R}}^t\mathbf{\tilde{p}}_R^{[0]}.
\end{equation}}

\begin{theorem}\label{Th_EA_Knapsack}
The expected approximation error of (1+1)EA for the Knapsack problem is bounded by
\begin{align*}
  e^{[t]}&\le\textstyle\left[(1-\alpha)n+1+\frac{1}{\alpha(1-\alpha)2^{\alpha n-1}}\left(1-\left(1-\frac{1}{n}\right)^{\alpha n-1}\right)\right]\left[1-\left(\frac{1}{n}\right)^{\alpha n}\right]^t\\
  &\quad +\textstyle \left\{\frac{(2^{\alpha n-2}-1)(\alpha n-1)}{2^{\alpha n-1}}+\frac{\alpha n-1}{2^{\alpha n-2}\alpha(1-\alpha)}\left[1-\left(1-\frac{1}{n}\right)^{\alpha n-1}\right]\right\}\left[1-\left(\frac{1}{n}\right)^{\alpha n}-\frac{1}{n}\left(1-\frac{1}{n}\right)^{\alpha n-2}\right]^{t}\\
  &\quad +\textstyle \left(n-\frac{1}{n}\right)\frac{1}{2^{\alpha n}}\left(1-\frac{1}{n}\right)^{\alpha n t}.
\end{align*}
\end{theorem}
\begin{proof}
Partition the transition matrix as
\begin{align}
  \mathbf{\tilde{R}}=\left(
                       \begin{array}{cc}
                         \mathbf{\hat{R}} & \mathbf{\hat{r}}^{[1]}  \\
                         0 & r_{\alpha n,\alpha n}  \\
                       \end{array}
                     \right)
\end{align}
where $\mathbf{\hat{R}}=(r_{i,j})_{\alpha n\times\alpha n}, i,j=0,1,\dots,\alpha n-1$, $\mathbf{\hat{r}}^{[1]}=(r_{0,\alpha n},r_{1,\alpha n},\dots,r_{\alpha n-1,\alpha n})'$. By applying Lemma \ref{Sec4L5} to (\ref{errEAKnap}) we know
\begin{align}
  e^{[t]}\le\mathbf{\hat{e}}_R'\mathbf{\hat{R}}^t\mathbf{\hat{p}}_R^{[0]}+(p_{\alpha n}^{[0]}+p_{\alpha n+1}^{[0]})\sum_{k=0}^{t-1}r_{\alpha n,\alpha n}^{k}\mathbf{\hat{e}}_R'\mathbf{\hat{R}}^{t-1-k}\mathbf{\hat{r}}^{[1]}+(p_{\alpha n}^{[0]}+p_{\alpha n+1}^{[0]})e_{\alpha n+1}r_{\alpha n,\alpha n}^t, \label{temp4}
\end{align}
where \begin{equation*}
\mathbf{\hat{e}}_R=(e_0,\dots,e_{\alpha n-1})'=\left(0, n-(\alpha n-1),\dots, n-1\right)',
\end{equation*}
\begin{equation*}\mathbf{\hat{p}}^{[0]}_R=(p^{[0]}_0,\dots,p^{[0]}_{\alpha n-1})'=\left(\frac{1}{2^{\alpha n}}, C_{\alpha n-1}^{\alpha n-1}\left(\frac{1}{2}\right)^{\alpha n-1},\dots,C_{\alpha n-1}^{1}\left(\frac{1}{2}\right)^{\alpha n-1}\right)'.
\end{equation*}
Furthermore, denote
\begin{equation}\label{erTheo11}
\mathbf{{e}}_R=(e_1,\dots,e_{\alpha n-1})'=\left( n-(\alpha n-1),\dots, n-1\right)',
\end{equation}
\begin{equation}\label{p0Theo11}
\mathbf{{p}}^{[0]}_R=(p^{[0]}_1,\dots,p^{[0]}_{\alpha n-1})'=\left(C_{\alpha n-1}^{\alpha n-1}\left(\frac{1}{2}\right)^{\alpha n-1},\dots,C_{\alpha n-1}^{1}\left(\frac{1}{2}\right)^{\alpha n-1}\right)',
\end{equation}
\begin{equation}\label{rTheo11}
  \mathbf{{r}}^{[1]}=(r_{1,\alpha n},\dots,r_{\alpha n-1,\alpha n})',
\end{equation}
and construct an auxiliary transition matrix
\begin{equation}\label{TranKnap3}
  \mathbf{\hat{S}}=(s_{i,j})_{\alpha n\times\alpha n}=\left(\begin{array}{cc}
                             1 & \mathbf{s} \\
                             0 & \mathbf{S}
                           \end{array}\right),\quad i,j=0,1,\dots,\alpha n-1,
\end{equation}
where $$\mathbf{s}=\left(\left(\frac{1}{n}\right)^{\alpha n}, \left(\frac{1}{n}\right)^{\alpha n},\cdots, \left(\frac{1}{n}\right)^{\alpha n}\right),$$
\begin{align}\label{TrKnapEA}
  &\mathbf{S}=(s_{i,j})_{i,j=1,\dots,\alpha n-1}\nonumber\\
  &\tiny \left(
                 \begin{array}{ccccc}
                      1-\left(\frac{1}{n}\right)^{\alpha n} & \frac{1}{n}\left(1-\frac{1}{n}\right)^{\alpha n-2} & \cdots  & 0 & 0  \\
                        & 1-\left(\frac{1}{n}\right)^{\alpha n}-\frac{1}{n}\left(1-\frac{1}{n}\right)^{\alpha n-2} & \cdots  & 0 & 0\\
                      &   &  \ddots  &\vdots &\vdots\\
                        &   &   & \frac{\alpha n-3}{n}\left(1-\frac{1}{n}\right)^{\alpha n-2} & 0\\
                          &   &   &   1-\left(\frac{1}{n}\right)^{\alpha n}-\frac{\alpha n-3}{n}\left(1-\frac{1}{n}\right)^{\alpha n-2} & \frac{\alpha n-2}{n}\left(1-\frac{1}{n}\right)^{\alpha n-2}\\
                               &   &   &   & 1-\left(\frac{1}{n}\right)^{\alpha n}-\frac{\alpha n-2}{n}\left(1-\frac{1}{n}\right)^{\alpha n-2}\\
                 \end{array}
               \right).
\end{align}
It is trivial to check that matrices $\mathbf{\hat{R}}$ and $\mathbf{\hat{S}}$ satisfy conditions (\ref{conC1})-(\ref{conC3}). Applying Theorem \ref{Th_ES} and Lemma \ref{Sec1L1} to (\ref{temp4}) we know
\begin{align}
  e^{[t]}&\le\mathbf{\hat{e}}_R'\mathbf{\hat{S}}^t\mathbf{\hat{p}}_R^{[0]}+(p_{\alpha n}^{[0]}+p_{\alpha n+1}^{[0]})\sum_{k=0}^{t-1}r_{\alpha n,\alpha n}^{k}\mathbf{\hat{e}}_R'\mathbf{\hat{S}}^{t-1-k}\mathbf{\hat{r}}^{[1]}+(p_{\alpha n}^{[0]}+p_{\alpha n+1}^{[0]})e_{\alpha n+1}r_{\alpha n,\alpha n}^t\nonumber\\
  &=\mathbf{{e}}_R'\mathbf{{S}}^t\mathbf{{p}}_R^{[0]}+(p_{\alpha n}^{[0]}+p_{\alpha n+1}^{[0]})\sum_{k=0}^{t-1}r_{\alpha n,\alpha n}^{k}\mathbf{{e}}_R'\mathbf{{S}}^{t-1-k}\mathbf{{r}}^{[1]}+(p_{\alpha n}^{[0]}+p_{\alpha n+1}^{[0]})e_{\alpha n+1}r_{\alpha n,\alpha n}^t \label{temp5}
\end{align}
Then, Lemma \ref{Sec2L3} implies that
\begin{align}
  e^{[t]}&\le\sum_{j=1}^{\alpha n-1}\lambda_j^t(\mathbf{{e}}_R'\mathbf{p}_j)(\mathbf{q}_j'\mathbf{{p}}_R^{[0]})+(p_{\alpha n}^{[0]}+p_{\alpha n+1}^{[0]})\sum_{k=0}^{t-1}r_{\alpha n,\alpha n}^{k}\sum_{j=1}^{\alpha n-1}\lambda_j^{t-1-k}(\mathbf{{e}}_R'\mathbf{p}_j)(\mathbf{q}_j'\mathbf{{r}}^{[1]})\nonumber\\
  &\qquad+(p_{\alpha n}^{[0]}+p_{\alpha n+1}^{[0]})e_{\alpha n+1}r_{\alpha n,\alpha n}^t, \label{temp6}
\end{align}
where
\begin{align}\label{lambddaTh11}
\lambda_{j}=s_{j,j}=1-\left(\frac{1}{n}\right)^{\alpha n}-\frac{j-1}{n}\left(1-\frac{1}{n}\right)^{\alpha n-2},
\end{align}
\begin{align}
  \mathbf{p}_{j}&=\left(\prod_{k=1}^{j-1}\frac{s_{k,k+1}}{s_{j,j}-s_{k,k}},\prod_{k=2}^{j-1}\frac{s_{k,k+1}}{s_{j,j}-s_{k,k}},\dots,\frac{s_{j-1,j}}{s_{j,j}-s_{j-1,j-1}},1,0,\dots,0\right)',\label{pTheo11}\\
 \mathbf{q}'_{j}&=\left(0,\dots,0,1,\frac{s_{j,j+1}}{s_{j,j}-s_{j+1,j+1}},\prod_{k=j+1}^{j+2}\frac{s_{k-1,k}}{s_{j,j}-s_{k,k}},\dots,\prod_{k=j+1}^{n}\frac{s_{k-1,k}}{s_{j,j}-s_{k,k}}\right),\label{qTheo11}
\end{align}
$j=1\dots,\alpha n-1$. Substituting (\ref{lambddaTh11}), (\ref{EPKnapTh11}), (\ref{QPKnapTh11}) and (\ref{QRTh11}) to (\ref{temp6}), we conclude that
\begin{align*}
  e^{[t]}&\le \textstyle (n-\alpha n+1)\left(1-\frac{1}{2^{\alpha n-1}}\right)\left[1-\left(\frac{1}{n}\right)^{\alpha n}\right]^t+(\alpha n-1)\left(\frac{1}{2}-\frac{1}{2^{\alpha n-1}}\right)\left[1-\left(\frac{1}{n}\right)^{\alpha n}-\frac{1}{n}\left(1-\frac{1}{n}\right)^{\alpha n-2}\right]^t\\
  &+\textstyle \frac{1}{2^{\alpha n}}\left[1-\left(1-\frac{1}{n}\right)^{\alpha n-1}\right]\left\{\sum_{k=0}^{t-1}\left[\left(1-\frac{1}{n}\right)^{\alpha n}\right]^k\left[1-\left(\frac{1}{n}\right)^{\alpha n}\right]^{t-1-k}\right.\\
  &\textstyle \left. +(\alpha n-1)\sum_{k=0}^{t-1}\left[\left(1-\frac{1}{n}\right)^{\alpha n}\right]^k\left[1-\left(\frac{1}{n}\right)^{\alpha n}-\frac{1}{n}\left(1-\frac{1}{n}\right)^{\alpha n-2}\right]^{t-1-k}\right\} +\left(n-\frac{1}{n}\right)\frac{1}{2^{\alpha n}}\left(1-\frac{1}{n}\right)^{\alpha n t}\\
  &\le \textstyle (n-\alpha n+1)\left(1-\frac{1}{2^{\alpha n-1}}\right)\left[1-\left(\frac{1}{n}\right)^{\alpha n}\right]^t+(\alpha n-1)\left(\frac{1}{2}-\frac{1}{2^{\alpha n-1}}\right)\left[1-\left(\frac{1}{n}\right)^{\alpha n}-\frac{1}{n}\left(1-\frac{1}{n}\right)^{\alpha n-2}\right]^t\\
  &+\textstyle \frac{1}{2^{\alpha n}}\left[1-\left(1-\frac{1}{n}\right)^{\alpha n-1}\right]\left\{\frac{2}{\alpha(1-\alpha)}\left[1-\left(\frac{1}{n}\right)^{\alpha n}\right]^{t} +\frac{4(\alpha n-1)}{\alpha(1-\alpha)}\left[1-\left(\frac{1}{n}\right)^{\alpha n}-\frac{1}{n}\left(1-\frac{1}{n}\right)^{\alpha n-2}\right]^{t}\right\} \\
  &+\textstyle \left(n-\frac{1}{n}\right)\frac{1}{2^{\alpha n}}\left(1-\frac{1}{n}\right)^{\alpha n t}\\
  &\le \textstyle\left[(1-\alpha)n+1+\frac{1}{\alpha(1-\alpha)2^{\alpha n-1}}\left(1-\left(1-\frac{1}{n}\right)^{\alpha n-1}\right)\right]\left[1-\left(\frac{1}{n}\right)^{\alpha n}\right]^t\\
  &\quad +\textstyle \left\{\frac{(2^{\alpha n-2}-1)(\alpha n-1)}{2^{\alpha n-1}}+\frac{\alpha n-1}{2^{\alpha n-2}\alpha(1-\alpha)}\left[1-\left(1-\frac{1}{n}\right)^{\alpha n-1}\right]\right\}\left[1-\left(\frac{1}{n}\right)^{\alpha n}-\frac{1}{n}\left(1-\frac{1}{n}\right)^{\alpha n-2}\right]^{t}\\
  &\quad +\textstyle \left(n-\frac{1}{n}\right)\frac{1}{2^{\alpha n}}\left(1-\frac{1}{n}\right)^{\alpha n t}.
\end{align*}

\end{proof}

\begin{figure}
  \centering
  \includegraphics[width=15cm]{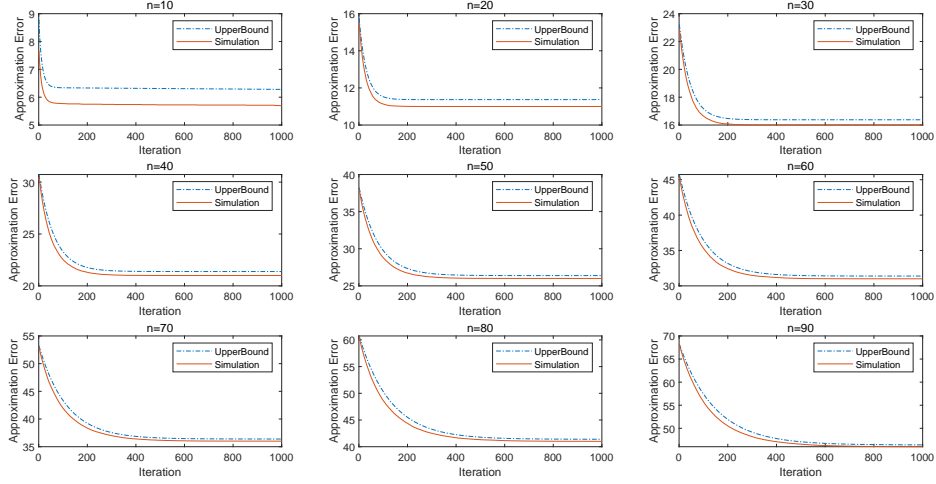}\\
  \caption{Comparison between the estimated upper bound and simulation results on expected approximation error of (1+1)EA solving the 10-, 20-,...,90-D Knapsack problems.}\label{fig3}
\end{figure}

As illustrated in Figure \ref{fig3}, it is showed that the estimated upper bound for the Knapsack problem is very tight, too. Although global exploration ability can be achieved by the bitwise mutation, the (1+1)EA always focuses on local exploitation. Then, the searching process is absorbed by the local optimal solution. Then, to jump from it to the global optimal solution is a difficult task because the transition probability is $\left(\frac{1}{n}\right)^{\alpha n}$. As a consequence, the upper bound of $e^{[t]}$ is dominated by the first item, which is of the order $\left[1-\left(\frac{1}{n}\right)^{\alpha n}\right]^t$.

}
\section{Conclusions and Discussions}\label{SecCon}

In order to bridge the gap between theories and applications of RSH, this paper is dedicated to analyze elitist RSH by estimating the expected approximation error for iteration budget $t$. According to the distribution of non-zero elements in the transition matrix of Markov chain, searching processes of elitist RSH are classified into three categories, and we propose a general framework for estimation of approximation error, named as the error analysis.

{\color{red} Since error analysis is based on computation the $t$-th power of the transition probability matrix, we can obtain general results on expected approximation error regarding any iteration $t$. Meanwhile, the obtained results are concrete expressions of approximation error, which is much more precise than the asymptotic results of fixed-budget analysis. Furthermore, the analysis routine can be theoretically applied to any RSH that is modeled by a upper triangular transition matrix, which demonstrates the universality of error analysis.} Tricks of error analysis are definition of statues, diagonalization of upper triangular matrices and multiplication of block matrices. With help of these mathematical techniques, the error analysis can be applied easily on analysis of elitist RHS for uni- and multi-modal problems.

{\color{red}Analysis of  population-based EAs in the framework of error analysis is feasible if we can address how the transition probability is influenced by population size. For the (1+$\lambda$) EA that generates multiple offsprings by one parent, it is easy to estimate improvement of transition probability, and the challenge lies in computation of the $t$-th power of transition matrix. However, to analyze ($N$+$N$)EA we must overcome the difficulties in estimation of transition probability and computation of the $t$-th  power.} There are some other open questions in error analysis, including construction of Markov chain model, design of auxiliary searches, and computation of combinatorics, etc. Moreover, the analyzing routine is based on the precondition that the transition matrix is diagonalizable. Thus, we would like to analyze RSH whose transition matrix is not diagonalizable.

%
%

\section*{Acknowledgements}
This work was partly supported by the Fundamental Research Funds for the Central Universities (WUT: 2020IB006), the National Nature Science Foundation of China under Grant 61763010, the Guangxi ``BAGUI Scholar'' Program, and the Science and Technology Major Project of Guangxi under Grant AA18118047.

\appendix
\section{Computation of $\mathbf{e}'\mathbf{p}_j$ and $\mathbf{q}_j\mathbf{p}^{[0]}$ in Proof of Theorem \ref{Th_RLS_OneMax}}\label{AppendixAA}
Denote $\mathbf p_{j}=(p_{1,j},\dots,p_{n,j})'$, $\mathbf q_{j}=(q_{1,j},\dots,q_{n,j})'$. Then, By Lemma \ref{Sec2L3} we can get the values of $p_{i,j}$ as follows.
\begin{enumerate}
  \item If $i>j$, $p_{i,j}=0$;
  \item if $i=j$, $p_{i,j}=1$;
  \item if $i<j$, equation (\ref{TrOneRLS}) implies that
  \begin{align*}
    p_{i,j}  =\prod_{k=i}^{j-1}\frac{r_{k,k+1}}{r_{j,j}-r_{k,k}}=\prod_{k=i}^{j-1}\frac{\frac{k+1}{n}}{\frac{k}{n}-\frac{j}{n}}=(-1)^{j-i}C_j^{i},\quad,i=1,\dots,j-1.
  \end{align*}
\end{enumerate}
Similarly, we can also get the value of $q_{i,j}$ as follows.
\begin{enumerate}
  \item If $i<j$, $q_{i,j}=0$;
  \item if $i=j$, $q_{i,j}=1$;
  \item if $i>j$, equation (\ref{TrOneRLS}) implies that
  \begin{align*}
    q_{i,j}  =\prod_{k=j+1}^{i}\frac{r_{k-1,k}}{r_{j,j}-r_{k,k}}=\prod_{k=j+1}^{i}\frac{\frac{k}{n}}{\frac{k}{n}-\frac{j}{n}}=C_i^j,\quad,i=j+1,\dots,n.
  \end{align*}
\end{enumerate}
In summary, we know
\begin{align}
  \mathbf{p}_{j}&=\left((-1)^{j-1}C_j^1,(-1)^{j-2}C_j^2,\dots,(-1)^{1}C_j^{j-1},1,0,\dots,0\right)',\label{p1}\\
  \mathbf{q}'_{j}&=\left(0,\dots,0,1,C_{j+1}^j,C_{j+2}^j,\dots,C_{n}^j\right),\label{q1}\\
  &\quad\quad\quad\quad\quad\quad\quad\quad\quad\quad\quad\quad\quad\quad\quad\quad\quad j=1,\dots,n.\nonumber
  \end{align}
Combing (\ref{ErOneMax}) and (\ref{p1}) we know
\begin{align}\label{temp1}
  \mathbf{e}'\mathbf{p}_j=\sum_{i=1}^ji(-1)^{j-i}C_{j}^{i}=\sum_{i=1}^j(-1)^{j-i}C_{j}^{1}C_{j-1}^{i-1}= \left\{\begin{aligned}& 1, && \mbox{ if }j=1;\\ & 0, && \mbox{ if } 2\le j\le n. \end{aligned}\right.
\end{align}
Moreover, (\ref{DiOneMax}) and (\ref{q1}) implies that
\begin{align}\label{temp2}
  \mathbf{q}'_1\mathbf{p}^{[0]}=\sum_{i=1}^{n}C_{i}^{1}\frac{C_{n}^{i}}{2^n}=\frac{1}{2^n}\sum_{i=1}^{^n}C_n^1C_{n-1}^{i-1}=\frac{n}{2}.
\end{align}
\section{Computation of $\mathbf{e}\mathbf{p}_j$ and $\mathbf{q}_j\mathbf{p}^{[0]}$ in Proof of Theorem \ref{Th_RLS_Deceptive}}\label{AppendixA}
Denote $\mathbf p_{j}=(p_{1,j},\dots,p_{n,j})'$, $\mathbf q_{j}=(q_{1,j},\dots,q_{n,j})'$. Then, By Lemma \ref{Sec2L3} we can get the values of $p_{i,j}$ as follows.
\begin{enumerate}
  \item If $i>j$, $p_{i,j}=0$;
  \item if $i=j$, $p_{i,j}=1$;
  \item if $i<j$, equation (\ref{TrDeRLS}) implies that
  \begin{align*}
    p_{i,j}=&\prod_{k=i}^{j-1}\frac{r_{k,k+1}}{r_{j,j}-r_{k,k}} \\ =&\left\{\begin{aligned}&\prod_{k=i}^{j-1}\frac{\frac{k}{n}}{\frac{k-1}{n}-\frac{j-1}{n}}=(-1)^{j-i}C_{j-1}^{i-1}, &&i=1,\dots,j-1, j=1,\dots, n-1;\\
    &\prod_{k=i}^{j-1}\frac{\frac{k}{n}}{0-(1-\frac{k-1}{n})}=(-1)^{j-i}C_{j-1}^{i-1}\frac{1}{n+1-i}, &&i=1,\dots,j-1, j=n.\end{aligned}
    \right.
  \end{align*}
\end{enumerate}
Similarly, we can  get the value of $q_{i,j}$ as follows.
\begin{enumerate}
  \item If $i<j$, $q_{i,j}=0$;
  \item if $i=j$, $q_{i,j}=1$;
  \item if $i>j$, equation (\ref{TrDeRLS}) implies that
  \begin{align*}
    q_{i,j} & =\prod_{k=j+1}^{i}\frac{r_{k-1,k}}{r_{j,j}-r_{k,k}}\\
    &=\left\{\begin{aligned}&\prod_{k=j+1}^{i}\frac{\frac{k-1}{n}}{\frac{k-1}{n}-\frac{j-1}{n}}=C_{i-1}^{j-1}, && i=j+1,\dots,n-1;\\ &\prod_{k=j+1}^{i-1}\frac{\frac{k-1}{n}}{\frac{k-1}{n}-\frac{j-1}{n}}\frac{n-1}{n-j+1}=C_{i-1}^{j-1}\frac{n-j}{n-j+1}, && i=n.\end{aligned}\right.
  \end{align*}
\end{enumerate}
That is,
\begin{align*}
  \mathbf{p}_{j}&=\left\{\begin{aligned}&\left((-1)^{j-1}C_{j-1}^0,(-1)^{j-2}C_{j-1}^1,\dots,-C_{j-1}^{j-2},1,0,\dots,0\right)^T,&& j<n\\&\left((-1)^{n-1}C_{n-1}^{1-1}/C_{n+1-1}^{1},(-1)^{n-2}C_{n-1}^{2-1}/C_{n+1-2}^{1},\dots,-C_{n-1}^{(n-1)-1}/C_{n+1-(n-1)}^{1},1\right)^T, && j=n. \end{aligned}\right. \\
  \mathbf{q}'_{j}&=\left(0,\dots,0,1,C_{j}^{j-1},C_{j+1}^{j-1},\dots,C_{n-2}^{j-1},C_{n-1}^{j-1}\frac{n-j}{n+1-j}\right), \quad j=1,\dots,n.
\end{align*}
By equation (\ref{ErDec}), we know when $j<n$,
\begin{align}
  \mathbf{e}\mathbf{p}_j&=\sum_{i=1}^{j}(-1)^{j-i}C_{j-1}^{i-1}i=\left.\frac{d}{dx}\left(\sum_{i=1}^{j}(-1)^{j-i}C_{j-1}^{i-1}x^i\right)\right|_{x=1}\nonumber\\
  &=\left.\frac{d}{dx}\left(x\sum_{k=0}^{j-1}(-1)^{j-1-k}C_{j-1}^{k}x^k\right)\right|_{x=1}=\left.\frac{d}{dx}\left(x(x-1)^{j-1}\right)\right|_{x=1}\nonumber\\
  &=\left\{\begin{aligned}&1, && j=1,2,\\ & 0, && j=3,\dots,n-1; \end{aligned}\right. \label{ep1}
\end{align}
and if $j=n$,
\begin{align}
  \mathbf{e}\mathbf{p}_n&=\sum_{i=1}^{n}(-1)^{n-i}iC_{n-1}^{i-1}/(n+1-i)=\sum_{i=1}^{n}(-1)^{n-i}C_{n-1}^{i-1}-(n+1)\sum_{i=1}^{n}(-1)^{n-i}C_{n-1}^{i-1}/(n+1-i)\nonumber\\
  &=-(n+1)\left.\left(\sum_{i=1}^{n}(-1)^{n-i}C_{n-1}^{i-1}\int_{0}^{x}x^{n-i}dx\right)\right|_{x=1}=-(n+1)\left.\int_{0}^{x}x^n\left(\sum_{i=1}^{n}(-1)^{n-i}C_{n-1}^{i-1}x^{-i}\right)dx\right|_{x=1}\nonumber\\
  &=-(n+1)\left.\int_{0}^{x}(x-1)^{n-1}dx\right|_{x=1}=(-1)^{n}\frac{n+1}{n}. \label{ep2}
\end{align}
By equation (\ref{DiDec}), we know that
\begin{align}
  \mathbf{{q}}'_j\mathbf{{p}}^{[0]}&=\sum_{i=j}^{n-1}C_{i-1}^{j-1}\frac{C_n^{i-1}}{2^n}+C_{n-1}^{j-1}\frac{n-j}{n+1-j}\frac{C_n^{n-1}}{2^n}\nonumber\\
  &=\frac{1}{2^n}\left(\sum_{i=j}^nC_n^{j-1}C_{n-j+1}^{i-j}-C_n^{j-1}\right)=\frac{C_n^{j-1}}{2^n}\left(2^{n-j+1}-2\right). \label{qq1}
\end{align}

\section{Estimation of $\mathbf{\check{e}}\mathbf{\check{p}_j}$ and $\mathbf{\check{q}}_j\mathbf{\check{p}}^{[0]}$ and $\mathbf{\check{q}}_j\mathbf{\check{r}}^{[1]}$ in Proof of Theorem \ref{Th_EA_Deceptive}}\label{AppendixB}
Similar to computation of (\ref{p1}) and (\ref{q1}), from (\ref{TrOneRLSTheo9}) we know
\begin{align*}
  \mathbf{\check{p}}_{j}&=\left((-1)^{j-1}C_{j-1}^0,(-1)^{j-2}C_{j-1}^1,\dots,-C_{j-1}^{j-2},1,0,\dots,0\right)',j=1,\dots,n-1 \\
  \mathbf{\check{q}}_j'&=\left(0,\dots,0,1,C_{j}^{j-1},C_{j+1}^{j-1},\dots,C_{n-2}^{j-1}\right), \quad j=1,\dots,n-1.
\end{align*}
By equations (\ref{ep1}) and (\ref{etheo9}), we know
\begin{align}
  \mathbf{\check{e}}\mathbf{\check{p}}_j=\sum_{i=1}^{j}(-1)^{j-i}C_{j-1}^{i-1}i=\left\{\begin{aligned}&1, && j=1,2,\\ & 0, && j=3,\dots,n-1. \end{aligned}\right. \label{ep22}
\end{align}

Moreover, equation (\ref{qtheo9}) implies
\begin{align}
  \mathbf{\check{q}}'_j\mathbf{\check{p}}^{[0]}=\sum_{i=j}^{n-1}C_{i-1}^{j-1}\frac{C_n^{i-1}}{2^n}=\frac{C_n^{j-1}}{2^n}(2^{n-j+1}-(n-j+2)), \label{qq2}
\end{align}
and by equation  (\ref{rtheo9}) we know
\begin{align}\label{qr1}
  \mathbf{\check{q}}_j'\mathbf{\check{r}}^{[1]}&=\sum_{i=j}^{n-1}C_{i-1}^{j-1}r_{i,n}\le C_{n-1}^{j-1}(1-r_{0,n}-r_{n,n})\nonumber\\
  &=C_{n-1}^{j-1}\left[1-\frac{1}{n}\left(1-\frac{1}{n}\right)^{n-1}-\left(\left(1-\frac{1}{n}\right)^{n}+C_{n-1}^1\left(\frac{1}{n}\right)^2\left(1-\frac{1}{n}\right)^{n-2}\right)\right]\nonumber\\
  &=C_{n-1}^{j-1}\left(1-\frac{n+1}{n}\left(1-\frac{1}{n}\right)^{n-1}\right).
\end{align}

\section{Computation of $\mathbf{\check{e}}\mathbf{\check{p}}_j$ and $\mathbf{\check{q}}_j\mathbf{\check{p}}^{[0]}$ in Proof of Theorem \ref{Th_RLS_Knapsack}}\label{AppendixC}
Deduction of $\mathbf{\check{e}}\mathbf{\check{p}}_j$ and $\mathbf{\check{q}}_j\mathbf{\check{p}}^{[0]}$ is similar to computation in \ref{AppendixA}. From (\ref{TrKnapRLS1}) we know
\begin{align*}
  \mathbf{\check{p}}_{j}&=\left\{\begin{aligned}&\left((-1)^{j-1}C_{j-1}^0,(-1)^{j-2}C_{j-1}^1,\dots,-C_{j-1}^{j-2},1,0,\dots,0\right)^T,&& j<\alpha n\\&\left((-1)^{j-1}C_{j-1}^{1-1}/C_{j+1-1}^{1},(-1)^{j-2}C_{j-1}^{2-1}/C_{j+1-2}^{1},\dots,-C_{j-1}^{(j-1)-1}/C_{j+1-(j-1)}^{1},1\right)^T, && j=\alpha n, \end{aligned}\right. \\
  \mathbf{\check{q}}'_{j}&=\left(0,\dots,0,1,C_{j}^{j-1},C_{j+1}^{j-1},\dots,C_{\alpha n-2}^{j-1},C_{\alpha n-1}^{j-1}\frac{\alpha n-j}{\alpha n+1-j}\right), \quad j=1,\dots,\alpha n.
\end{align*}

By equation (\ref{eTh10}), we know when $j<\alpha n$,
\begin{align}
  \mathbf{\check{e}}\mathbf{\check{p}}_j&=\sum_{i=1}^{j}(-1)^{j-i}C_{j-1}^{i-1}(n-(\alpha n-i))=(1-\alpha)n\sum_{i=1}^{j}(-1)^{j-i}C_{j-1}^{i-1}+\sum_{i=1}^{j}(-1)^{j-i}C_{j-1}^{i-1}i\nonumber\\
  &=\left\{\begin{aligned}&n-\alpha n+1, && j=1,\\ & 1, && j=2,\\ & 0, && j=3,\dots,\alpha n-1; \end{aligned}\right. \label{ep1Appendd}
\end{align}
and if $j=\alpha n$,
\begin{align}
  \mathbf{\check{e}}\mathbf{\check{p}}_{\alpha n}&=\sum_{i=1}^{\alpha n}(-1)^{\alpha n-i}(n-(\alpha n-i))C_{\alpha n-1}^{i-1}/(\alpha n+1-i)\nonumber\\
  &=(n-1)\sum_{i=1}^{\alpha n}(-1)^{\alpha n-i}C_{\alpha n-1}^{i-1}/(\alpha n+1-i)-\sum_{i=1}^{\alpha n}(-1)^{\alpha n-i}C_{\alpha n-1}^{i-1}\nonumber\\
  &=(-1)^{\alpha n}(n-1)\frac{\alpha n+1}{\alpha n}. \label{ep2Appendd}
\end{align}
By equation (\ref{qTh10}), we know that when $j<\alpha n$,
\begin{align}
  \mathbf{\check{q}}'_j\mathbf{\check{p}}^{[0]}&=\sum_{i=j}^{\alpha n-1}C_{i-1}^{j-1}\frac{C_{\alpha n-1}^{i-1}}{2^{\alpha n-1}}+C_{\alpha n-1}^{j-1}\frac{\alpha n-j}{\alpha n+1-j}\left(\frac{1}{2^{\alpha n}}-\frac{1}{2^n}\right)\nonumber\\
  &=C_{\alpha n-1}^{j-1}\left[\frac{1}{2^{j-1}}-\frac{1}{2^{\alpha n-1}}+\frac{\alpha n-j}{\alpha n+1-j}\left(\frac{1}{2^{\alpha n}}-\frac{1}{2^n}\right)\right]. \quad j=1,\dots,\alpha n-1,\label{qq1Appendd}
\end{align}
and if $j=\alpha n$,
\begin{equation}\label{qq2Appendd}
  \mathbf{\check{q}}'_{\alpha n}\mathbf{\check{p}}^{[0]}=\frac{1}{2^{\alpha n}}-\frac{1}{2^n}.
\end{equation}
Moreover, equation (\ref{rrKnapRLS}) implies that if $j<\alpha n$,
\begin{equation}\label{qr1Th10}
  \left(1-\frac{1}{n}\right)C_{n-2}^{j-1}\frac{\alpha n-j}{\alpha n+1-j}\le\mathbf{\check{q}}_j'\mathbf{\check{r}}^{[1]} \le \left(1-\frac{1}{n}\right)C_{n-1}^{j-1},
\end{equation}
and if $j=\alpha n$,
\begin{equation}\label{qr2Th10}
  \mathbf{\check{q}}'_{\alpha n}\mathbf{\check{r}}^{[1]}=1-\alpha.
\end{equation}

\section{Computation of $\mathbf{e}_R'\mathbf{p}_j$, $\mathbf{q}_j'\mathbf{p}_R^{[0]}$ and $\mathbf{{q}} _j'\mathbf{{r}}^{[1]}$ in Proof of Theorem \ref{Th_EA_Knapsack} }\label{AppendixD}
Similar computation in \ref{AppendixAA}, by (\ref{pTheo11}) and (\ref{qTheo11}) we know
\begin{align*}
  \mathbf{p}_{j}&=\left((-1)^{j-1}C_{j-1}^0,(-1)^{j-2}C_{j-1}^1,\dots,-C_{j-1}^{j-2},1,0,\dots,0\right)',\\
  \mathbf{q}'_{j}&=\left(0,\dots,0,1,C_{j}^{j-1},C_{j+1}^{j-1},\dots,C_{\alpha n-2}^{j-1}\right),
\end{align*}
$j=1,\dots,\alpha n-1$.
Combining them with (\ref{erTheo11}), (\ref{p0Theo11}) and (\ref{rTheo11}), we know
 \begin{align}
\mathbf{e}_R'\mathbf{p}_j&=\sum_{i=1}^{j}(-1)^{j-i}C_{j-1}^{i-1}(n-(\alpha n-i))=\left\{\begin{aligned}& n-\alpha n+1, && j=1;\\ & 1, && j=2;\\& 0, && j=3,\dots,\alpha n-1,\end{aligned}\right. \label{EPKnapTh11}\\
  \mathbf{q}_j'\mathbf{p}_R^{[0]}&=\sum_{k=j}^{\alpha n-1}C_{k-1}^{j-1}C_{\alpha n-1}^{k-1}\frac{1}{2^{\alpha n-1}}=C_{\alpha n-1}^{j-1}\left[\frac{1}{2^{j-1}}-\frac{1}{2^{\alpha n-1}}\right].\label{QPKnapTh11}
\end{align}
\begin{align}
  \mathbf{{q}}_j'\mathbf{{r}}^{[1]}&=\sum_{i=j}^{\alpha n-1}C_{i-1}^{j-1}r_{i,\alpha n}\le\sum_{i=1}^{\alpha n-1}C_{i-1}^{j-1}r_{i,\alpha n}\le C_{\alpha n-1}^{j-1}(1-r_{0,\alpha n}-r_{\alpha n,\alpha n})\nonumber\\
  &=C_{\alpha n-1}^{j-1}\left[1-\left(1-\frac{1}{n}\right)^{\alpha n-1}\right].\label{QRTh11}
\end{align}
\bibliographystyle{elsarticle-num}
\bibliography{paper-E5}
\end{document}